\documentclass{article}

\usepackage{microtype}
\usepackage{graphicx}

\usepackage{subcaption}
\usepackage{booktabs} 
\usepackage{enumitem}
\usepackage{hyperref}

\usepackage[noend]{algorithmic}

\usepackage[preprint]{icml2026}

\usepackage{amsmath}
\usepackage{amssymb}
\usepackage{mathtools}
\usepackage{amsthm}

\usepackage[capitalize,noabbrev]{cleveref}

\theoremstyle{plain}
\newtheorem{theorem}{Theorem}[section]
\newtheorem{proposition}[theorem]{Proposition}
\newtheorem{lemma}[theorem]{Lemma}

\theoremstyle{definition}
\newtheorem{definition}[theorem]{Definition}

\theoremstyle{remark}
\newtheorem{remark}[theorem]{Remark}

\usepackage[textsize=tiny]{todonotes}

\icmltitlerunning{Unsupervised Decomposition and Recombination with Discriminator-Driven Diffusion Models}

\begin{document}

\twocolumn[
\icmltitle{Unsupervised Decomposition and Recombination with \\ Discriminator-Driven Diffusion Models}




\begin{icmlauthorlist}
\icmlauthor{Archer Wang}{rle,iaifi}
\icmlauthor{Emile Anand}{gt}
\icmlauthor{Yilun Du}{harvard,gdm}
\icmlauthor{Marin Soljacic}{rle,mit}

\end{icmlauthorlist}

\icmlaffiliation{rle}{Research Laboratory of Electronics, MIT, Cambridge, MA 02139, USA}

\icmlaffiliation{mit}{Department of Physics, Massachusetts Institute of Technology, MIT, Cambridge, MA, USA}

\icmlaffiliation{iaifi}{NSF AI Institute for Artificial Intelligence and Fundamental Interactions, Cambridge, MA 02139, USA}

\icmlaffiliation{harvard}{Kempner Institute, Harvard University, Cambridge, MA, USA}

\icmlaffiliation{gdm}{Google DeepMind, Mountain View, CA, USA}

\icmlaffiliation{gt}{School of Computer Science, Georgia Institute of Technology, Atlanta, USA}

\icmlcorrespondingauthor{Archer Wang}{archerdw.@mit.edu}


\icmlkeywords{Compositional Generation, Recombined Latents, Diffusion Models, Discriminators, Static Images, Robotic Videos, Machine Learning, ICML}

\vskip 0.3in
]



\printAffiliationsAndNotice{}  

\begin{abstract}
Decomposing complex data into factorized representations can reveal reusable components and enable synthesizing new samples via component recombination. We investigate this in the context of diffusion-based models that learn factorized latent spaces without factor-level supervision. In images, factors can capture background, illumination, and object attributes; in robotic videos, they can capture reusable motion components. To improve both latent factor discovery and quality of compositional generation, we introduce an adversarial training signal via a discriminator trained to distinguish between single-source samples and those generated by recombining factors across sources. By optimizing the generator to fool this discriminator, we encourage physical and semantic consistency in the resulting recombinations. Our method outperforms implementations of prior baselines on CelebA-HQ, Virtual KITTI, CLEVR, and Falcor3D, achieving lower FID scores and better disentanglement as measured by MIG and MCC. Furthermore, we demonstrate a novel application to robotic video trajectories: by recombining learned action components, we generate diverse sequences that significantly increase state-space coverage for exploration on the LIBERO benchmark.\looseness=-1


\end{abstract}

\section{Introduction}
\label{introduction}

\textit{Representation learning} and \textit{generative modeling} form two complementary pillars of machine learning. On one side, representation learning seeks to encode complex data into lower-dimensional latent spaces that capture the essential factors of variation; on the other, generative modeling maps latent codes back into the data space to synthesize new examples. Their intersection has given rise to powerful frameworks such as variational auto-encoders, which simultaneously learn a latent representation $z$ and a conditional likelihood $p(x|z)$ \cite{vae}. VAEs have spurred significant work on disentanglement, the goal of recovering latent factors that align with meaningful, independent attributes of the data \cite{Huang_2018_ECCV,Kim2018DisentanglingBF, infogan}. However, fully unsupervised disentanglement is provably impossible without additional assumptions \cite{Locatello2018ChallengingCA}, and numerous efforts have explored how to incorporate suitable inductive biases.
Beyond classic VAEs, recent research highlights the importance of \emph{compositional} structure in data generation. In many domains, particularly vision, observed data arise from combinations of discrete or continuous factors (objects, attributes, positions, etc.) \cite{misino2022vael,cosmos}. For instance, in a dataset of boats drifting along a river with trees in the background, a compositional model can discover components that capture the boat, the water/flow, and the scenery without explicit supervision; using these components, it can recombine factors across inputs to synthesize new realistic scenes of the same type. We focus on datasets where a finite set of factors explains most variation, such as faces with different hairstyles or expressions, or simple scenes of colored blocks. 

Exploiting such compositionality can yield stronger generalization, as models can reconfigure learned factors into novel configurations outside the training set via recombining components from different data points \cite{NEURIPS2022_5e6cec2a}. Probing the compositional structure has been of increasing interest to fields such as causal inference and ICA \cite{Seigal2022LinearCD, bansal2024universal}. Overall, learning better representations that reflect the generation process can be useful for downstream tasks such as classification, regression, visualization, and for building better generative models \cite{causalrepresentation}. \looseness=-1
Latent recombination is both useful and revealing: coherent recombinations suggest modular factors, while artifacts indicate entanglement. Ideally, such recombinations would be evaluated by external feedback (human preference, task success, or physical execution) which would directly refine the representation.

In summary, we make the following contributions:
First, we introduce an adversarial training signal for factorized latent diffusion that improves both factor discovery and compositional generation. 
Second, we demonstrate consistent gains over prior baselines on multiple image datasets, with improved recombination quality and stronger disentanglement.
Finally, we apply the approach to a robotics application: recombining learned action components yields diverse video trajectories that substantially increase state-space coverage for exploration on the LIBERO benchmark.




\section{Related Work}
\textbf{Learning compositional visual representations for images with diffusion models.} Diffusion-based models have emerged as the state-of-the-art generative models for images \cite{ddpm}, and recent works have investigated using this class of generative models to learn representations \cite{ddim,10.1007/978-3-031-19790-1_26,ftl-igm}. Previous works have also explored their use in compositional generation, such as combining visual concepts \cite{liu2021learning,du2023reduce,pmlr-v162-janner22a,10.5555/3495724.3496328} and robotic skills \cite{ajay2023is,10.5555/3692070.3694630}. The most relevant work is Decomp Diffusion \citep{decompdiffusion}, which is an unsupervised method that, when given an image, infers a set of different components in the image, each represented by a diffusion model. We build on Decomp Diffusion as a strong baseline for learning compositional components with diffusion models, and study how additional feedback during training can improve recombination quality and representation structure.\looseness=-1


Recent work has shown that using video generation to guide robotic control policy achieves great success \cite{unipi,avdc}. We conduct an additional experiment which focuses on videos of task demonstrations by robotics, and we study how good representations lead to generated recombined videos with better physical realism and also lead to the agent exploring its environment much more effectively with a pre-trained policy \cite{10.1007/978-3-030-86523-8_5,luo2024grounding,chi2023diffusionpolicy,zhang2024largescalereinforcementlearningdiffusion}.\looseness=-1


\textbf{Disentanglement and feedback signals.}
Disentanglement aims to represent data in terms of latent factors that can be independently manipulated, so that recombining components yields novel yet coherent samples \cite{pmlr-v80-denton18a,Brooks_2023_CVPR}. This property is closely tied to compositional generalization: if each latent encodes a distinct generative mechanism, then swapping one component should behave like an intervention on that mechanism while keeping the rest fixed. Traditional unsupervised disentanglement methods encourage factorial latents via regularization (e.g., $\beta$-VAE; FactorVAE); however, fully unsupervised disentanglement is ill-posed without additional assumptions: many different latent parameterizations can explain the same observational distribution \citep{betavae, Kim2018DisentanglingBF, Locatello2018ChallengingCA,black2024zeroshot}. In practice, identifying the ``right'' compositional structure often requires external feedback that reveals whether a recombination corresponds to a plausible intervention. Complementary lines of work study identifiability through structure and interventions, showing that interventional data can recover meaningful factors under suitable conditions \citep{Seigal2022LinearCD,causalrepresentation}. Feedback can be incorporated by fine-tuning diffusion models with reward gradients \cite{wang2025finetuning,zekri2025finetuning}.
ReFL \cite{xu2023imagereward} evaluates a reward on the one-step clean prediction $\hat{x}_0$ at a randomly chosen denoising step $t$ (rather than only on the final sample), while DRaFT \cite{draftk} backpropagates reward gradients through the sampling procedure.
Our feedback mechanism is inspired by the ReFL-style use of intermediate $\hat{x}_0$ predictions.\looseness=-1

\section{Background}

\begin{figure*}[ht]
    \centering
    \includegraphics[width=0.8\textwidth]{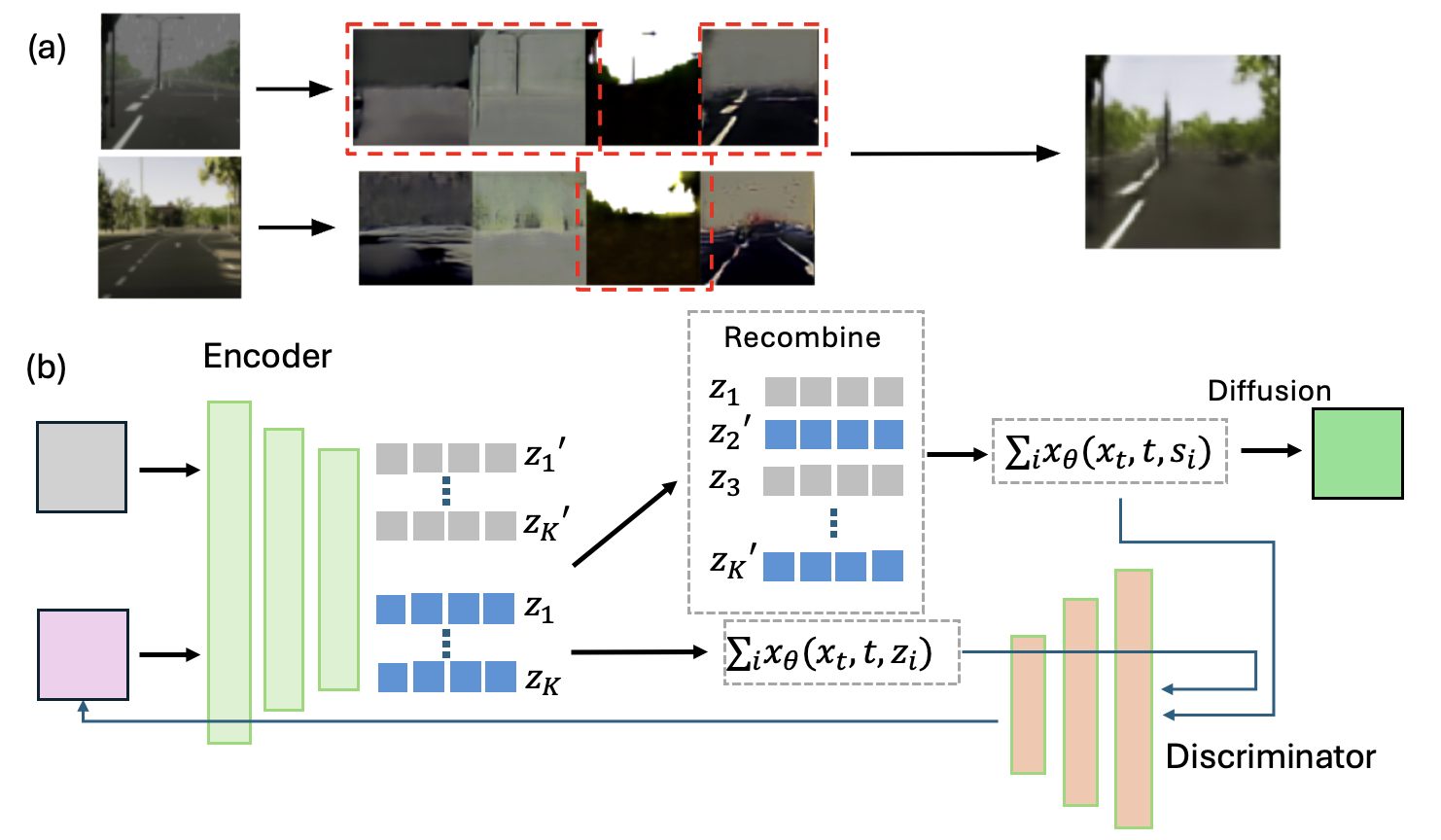} 
    \caption{
    (a) We encode two input images into $K=4 $ latent components. The middle panels show component-wise reconstructions: for each $k$, we decode using only $z_k$, with all other components set to zero. Red dashed boxes mark which components are chosen from each source to form a hybrid latent code $\tilde{z}$. Decoding $\tilde{z}$ yields the final image (right), which merges appearance/scene attributes from both sources. (b) During training, the discriminator learns to distinguish between predictions from latents from a single input and latents from multiple sources, while the model attempts to fool the discriminator. At inference, recombined latents are sampled via the standard diffusion denoising process.}
    
    \label{fig:discriminator_diagram}
\end{figure*}

Compositional modeling can be learned through
$$p_{\theta}(x)\propto \prod_i p_{\theta}^i(x).$$
Given a narrow slice of training, this compositional formulation allows us to generalize to unseen combinations of variables \citep{composition_position}, and this factorization can be traced back to the field of probabilistic graphical modeling \citep{koller2009}.


Given an input image or video of a scene, we model it by encoding it into several low-dimensional latent components, and we encourage each component to capture a different concept so that, together, they can be used to generate or score the input. \citet{du2019implicit} showed how sampling through the joint distribution on all latents is represented by generation on an
energy-based model (EBM) that is the sum of each conditional EBM, but this required supervision, as labels were used to represent concepts. \looseness=-1

COMET~\cite{du2021unsupervised} learns an energy-based model framework to decompose and discover these concepts in an unsupervised manner. 
Given an input image $x_i$, an encoder produces $K$ latent factors $\{z_k\}_{k=1}^K$, each
parameterizing an energy function $E_\theta(x; z_k)$.
COMET is trained so that $x_i$ corresponds to the minimum-energy configuration
under the sum of these energies:
\[
\mathcal{L}_{\mathrm{COMET}}(\theta)
\;=\;
\big\|
\arg\min_x \sum_{k} E_\theta(x; z_k)
\;-\; x_i
\big\|^2.
\]
In practice, the minimizer is approximated by a finite number of gradient descent steps on $x$,
making composition equivalent to joint energy minimization over multiple inferred factors. This formulation enables flexible recombination by mixing latent factors across images, but
requires iterative optimization at generation time. 

Recent diffusion-based approaches can be
viewed as replacing explicit energy minimization with denoising dynamics that implicitly follow
additive energy gradients, yielding a more stable and scalable mechanism for compositional
generation \cite{wu2023slotdiffusion,kirilenko2024objectcentric, zhu2024learning}. Decomp Diffusion links EBM composition to diffusion by decomposing the denoiser’s prediction into a mean over factor-specific contributions. Concretely, it mirrors the EBM decomposition $E_{\theta}(x)=\sum_i E_{\theta}(x; z_i)$ with $\epsilon_{\theta}(x_t,t)=\frac{1}{K}\sum_{k=1}^{K}\epsilon_{\theta}(x_t,t,z_k)$, where $z_k$ are learned latent factors, $K$ is the number of factors, and $\epsilon_{\theta}$ denotes the denoising network. In their best performing model, the authors demonstrate a noticeable improvement in disentanglement compared to $\beta$-VAE \cite{pmlr-v108-khemakhem20a,betavae,burgess2018understandingdisentanglingbetavae}, InfoGAN \cite{infogan}, GENESIS-V2 \cite{genesisv2}, and COMET \cite{du2021unsupervised}. \looseness=-1

\section{Compositional Generation with Discriminator}
Without explicit interventions during training, unsupervised generative models provide no control over the structure of their learned representations or over the behavior of latent recombination. \citet{additive_decoders} argues that for additive models, recombined samples may not always lie on the manifold of reasonable observations, and a theoretical characterization of these models is also done in \citet{wiedemer}, for a supervised setting.
Our setting follows this idealized perspective, but replaces expensive environment feedback with a scalable surrogate. We use a convolutional discriminator, whose locality and translation-equivariance make it efficient at detecting the low-level inconsistencies that often arise under recombination, thereby providing a pragmatic proxy for perceptual plausibility.\looseness=-1

We train the discriminator to distinguish single-source generations from recombined generations, and use its gradients to refine the encoder/generator so that recombinations become indistinguishable from valid samples. While this does not provide identifiability guarantees in the sense of causal disentanglement, it operationalizes the same intuition: recombination reveals whether factors are independently manipulable, and discriminator feedback supplies a practical signal for improving compositional structure.
A key challenge in compositional generative modeling is balancing compositional generalization with the plausibility of newly composed samples \cite{bagautdinov2018modeling}. Naively recombining latent factors often yields out-of-distribution generations, whereas overly strong constraints can collapse diversity. Our method introduces a tunable discriminator-driven feedback signal that selectively penalizes implausible recompositions while preserving compositional diversity.\looseness=-1

\subsection{Factorized Generative Model} Let $\mathcal{X}$ denote the data space (e.g. images or videos), and $\mathcal{M} \subset \mathcal{X}$ the data manifold corresponding to the support of the true distribution $p_{\text{data}}$. 
We assume each observation $x \in \mathcal{X}$ is generated by a factorized latent representation $z = (z_1, \ldots, z_K)$, where each $z_i \in \mathcal{Z}_i$ corresponds to a distinct generative factor (e.g. object shape, position, color). 
The generative model is described by
\begin{equation}
    x \sim p_\theta(\cdot\mid z), \quad z \sim p(z) = \prod_{i=1}^K p(z_i),
\end{equation}
and induces a model distribution $p_\theta(x)$ over $\mathcal{X}$. During training, we learn both an encoder $\text{Enc}_\phi : \mathcal{X} \to \mathcal{Z}_1 \times \cdots \times \mathcal{Z}_K$ and a conditional distribution $p_\theta(\cdot | z): \mathcal{Z} \to \Delta(\mathcal{X})$ by minimizing a reconstruction-based diffusion loss $\mathcal{L}_{\text{MSE}}(\theta, \phi)$.\looseness=-1

\subsection{Compositional Recombination} At inference time, we construct \emph{compositional} latent codes by recombining factors from multiple inputs:
\begin{equation}
    \tilde{z} =S \odot z^{(A)} + (1 - S) \odot z^{(B)}, \quad S \in \{0,1\}^K,
\end{equation}
where $\odot$ denotes element-wise multiplication and $z^{(A)}, z^{(B)}$ are latent representations of two source samples and $S$ is a binary mask indicating which factors to swap. 
The resulting generation $\tilde{x} \sim p_\theta(\cdot | \tilde{z})$ may correspond to a combination of factors that was never observed jointly in $p_{\text{data}}$, and thus $\tilde{x}$ may lie off the manifold $\mathcal{M}$.\looseness=-1

\subsection{Discriminator Feedback}
To discourage unrealistic recombinations, we introduce a recombination discriminator
$D_\psi : \mathcal{X} \to [0,1]$ trained to distinguish single-source generations from
recombined generations. Ideally, single-source samples are obtained by decoding latents inferred from a single data point and recombined
samples are constructed by mixing latent factors across multiple sources and decoding them. \looseness=-1

Formally, given two data points $x^A$ and $x^B$ and timestep $t$, let 
$\hat{x}^{\text{single}} = G_{\theta}(x^A_t, t,z_{1:K})$
and
$\hat{x}^{\text{recomb}}= G_{\theta}(x^A_t, t,\tilde{z}_{1:K})$, where denoising predictions come from contributions associated with each latent, we have $G_{\theta}(x,t,z_{1:K})=\frac{1}{K}\sum_{j=1}^K G_{\theta}(x,t,z_j)$.
We introduce a discriminator $D_\psi$ that distinguishes single-source generations
from recombined ones.
The discriminator is trained by minimizing the classification loss\looseness=-1
\begin{align*}
\mathcal{L}_{\mathrm{clf}}(\psi)
=
-\mathbb{E}\!\left[
\log D_\psi(\hat{x}^{\text{single}})
+
\log\!\left(1 - D_\psi(\hat{x}^{\text{recomb}})\right)
\right].
\end{align*}

The generator and encoder are trained adversarially to make recombined samples
indistinguishable from single-source samples by minimizing
\begin{align*}
\mathcal{L}_{\mathrm{adv}}(\theta,\phi;\psi)
=
-\mathbb{E}\!\left[\log D_\psi(\hat{x}^{\text{recomb}})\right].
\end{align*}
We train the model using alternating minimization.
The generator and encoder are optimized by minimizing
\begin{align*}
\mathcal{L}_{\mathrm{gen}}(\theta,\phi)
=
\mathcal{L}_{\mathrm{MSE}}(\theta,\phi)
+
\lambda
\mathcal{L}_{\mathrm{adv}}(\theta,\phi;\psi),
\end{align*}
while the discriminator is trained to minimize $\mathcal{L}_{\mathrm{clf}}(\psi)$. Here $\mathcal{L}_{\mathrm{MSE}}$ denotes the standard diffusion reconstruction loss
used to train the base model.\looseness=-1
 
\begin{algorithm}[btp]
\caption{Compositional image training with single-source vs recombined discriminator}
\label{alg:image-train}
\begin{algorithmic}[1]
\STATE \textbf{Input:} dataset $\mathcal{D}$, encoder $E_{\phi}$, generator/decoder $G_{\theta}$, discriminator $D_{\psi}$, weight $\lambda$
\WHILE{not converged}
  \STATE Sample two datapoints $x^{A}, x^{B} \sim \mathcal{D}$ and a recombination mask $S \in \{0,1\}^{K}$
  \STATE Encode latents: $z^{A}_{1:K} \leftarrow E_{\phi}(x^{A}),\;\; z^{B}_{1:K} \leftarrow E_{\phi}(x^{B})$
  \STATE Recombine latents: $\tilde{z}_{k} \leftarrow S_k z^{A}_k + (1-S_k)z^{B}_k$ for $k=1,\dots,K$
  \STATE Sample $t \sim \mathrm{Unif}(\{1,\dots,T\})$, $\epsilon\sim\mathcal{N}(0,I)$
  \STATE Forward diffusion: $x_t^{A} \leftarrow \sqrt{\bar{\alpha}_t}\,x^{A} + \sqrt{1-\bar{\alpha}_t}\,\epsilon$
  \STATE \textbf{Single-source generation:} $\hat{x}^{\text{single}} \!\leftarrow\!G_{\theta}(x^A_t,t,z^{A}_{1:K})$
  \STATE \textbf{Recombined generation:} $\hat{x}^{\text{recomb}} \!\leftarrow\!G_{\theta}(x^A_t,t,\tilde{z}_{1:K})$
  \STATE Compute diffusion reconstruction loss $\mathcal{L}_{\text{rec}} = \|x^{A} - \hat{x}^{\,\text{single}}\|_2^2$
  \STATE \textbf{Discriminator update (single-source/recombined):}
  \STATE 
$\mathcal{L}(\psi)
\!=\! -\mathbb{E}\!\left[
\log D_\psi(\hat{x}^{\text{single}})
\!+\! \log\!\left(1 - D_\psi(\hat{x}^{\text{recomb}})\right)
\right]$
  \STATE \hspace{1em} $\psi \leftarrow \psi - \eta_{\psi}\nabla_{\psi}\mathcal{L}$
  \STATE \textbf{Generator/encoder update:}
  \STATE \hspace{1em} $\mathcal{L}_{\text{adv}}(\theta,\phi) \;=\; -\mathbb{E}\!\left[\log D_{\psi}(\hat{x}^{\,\text{recomb}})\right]$
  \STATE \hspace{1em} $\theta,\phi \leftarrow \theta,\phi - \eta\nabla_{\theta,\phi}\left(\mathcal{L}_{\text{rec}} + \lambda\mathcal{L}_{\text{adv}}\right)$
\ENDWHILE
\end{algorithmic}
\end{algorithm}

\section{Experiments}
In this section, we evaluate the effectiveness of our proposed approach across multiple settings\footnote{
We provide source code for our experiments
\href{https://anonymous.4open.science/r/gan_project-5ED7}{here}
and
\href{https://anonymous.4open.science/r/video_2024-A49B}{here}.
}.
\subsection{Synthetic Recombination Experiment: Isolating Discriminator Feedback}
\label{sec:synthetic}

We first study a synthetic setting designed to isolate the effect of
discriminator feedback on latent recombination, independent of diffusion training
or high-capacity generators. The goal of this experiment is to verify our method's core mechanism: that
penalizing recombined samples provides a signal that reshapes representations
toward a recombination-friendly factorization.\looseness=-1

\textbf{Setup.}
We construct data from three independent latent factors
$s = (s_1, s_2, s_3)$, for $s_i \in \mathbb{R}^2$ drawn from distinct
non-Gaussian distributions. Observations $x \in \mathbb{R}^4$ are generated by a
fixed decoder with explicit cross-factor interactions; concretely, writing
$s_i=(s_{i,x}, s_{i,y})$, we define
$x_0=\langle s_1,s_2\rangle$, $x_1=\langle s_2,s_3\rangle$,
$x_2=\|s_1\|_2^2-\|s_3\|_2^2$, and $x_3=s_{1,x}s_{3,y}-s_{1,y}s_{3,x}$.
To simulate entangled representations, we apply a random linear mixing $z = Ms$.\looseness=-1

We introduce a discriminator trained to distinguish single-source samples from
recombined samples. A linear reparameterization $W$ of the latent space is then
trained adversarially to fool the discriminator on recombined samples, while
the decoder remains fixed. This setting corresponds to an idealized regime where the reconstruction objective is trivially satisfied, allowing us to
isolate the influence of the adversarial recombination loss.\looseness=-1

\textbf{Results.} As shown in Fig.~\ref{fig:synthetic_scatter}, naive recombination in the entangled coordinates produces samples that deviate substantially from the data manifold,
while discriminator feedback restores the geometry of real data. Quantitatively, recombined samples exhibit a large increase in off-manifold
distance relative to real data, which collapses back to near-real levels after discriminator training (Table~\ref{tab:synthetic_metrics}). Importantly, the learned reparameterization need not recover the true generative factors; instead, it reduces statistical dependence between latent components, enabling approximate closure under coordinate-wise recombination.
This toy experiment demonstrates that feedback from recombined samples alone can induce a representation that supports compositional recombination, even without semantic identifiability or diffusion training. We provide intuition for this behavior with our theoretical analysis (Appendix \ref{appendix:recomb}), which shows that recombination closure implies a Cartesian-product structure in latent space and that penalizing recombined samples contracts cross-component dependence.\looseness=-1

While this synthetic setting does not admit semantic identifiability, we can
quantify improved factorization using standard disentanglement-style metrics,
leveraging access to ground-truth factors.
Specifically, we report the mean correlation coefficient (MCC), which is widely used in the nonlinear ICA literature \cite{mcc}; the mutual
information gap (MIG, computed via binned mutual information) \cite{mig}, as measures of disentanglement with respect to the ground-truth factors. Separately, we report Gaussian total correlation (TC; a log-det covariance proxy) as a measure of statistical dependence among coordinates. To assess block-level independence, we also report the mean absolute off-diagonal correlation of per-block $\ell_2$ norms. We compute these metrics on the raw entangled latents $z$ before training and on
the discriminator-induced reparameterization $\hat{s}=Wz$ after training.\looseness=-1

\begin{figure}[h!]
  \centering
  \includegraphics[width=\linewidth]{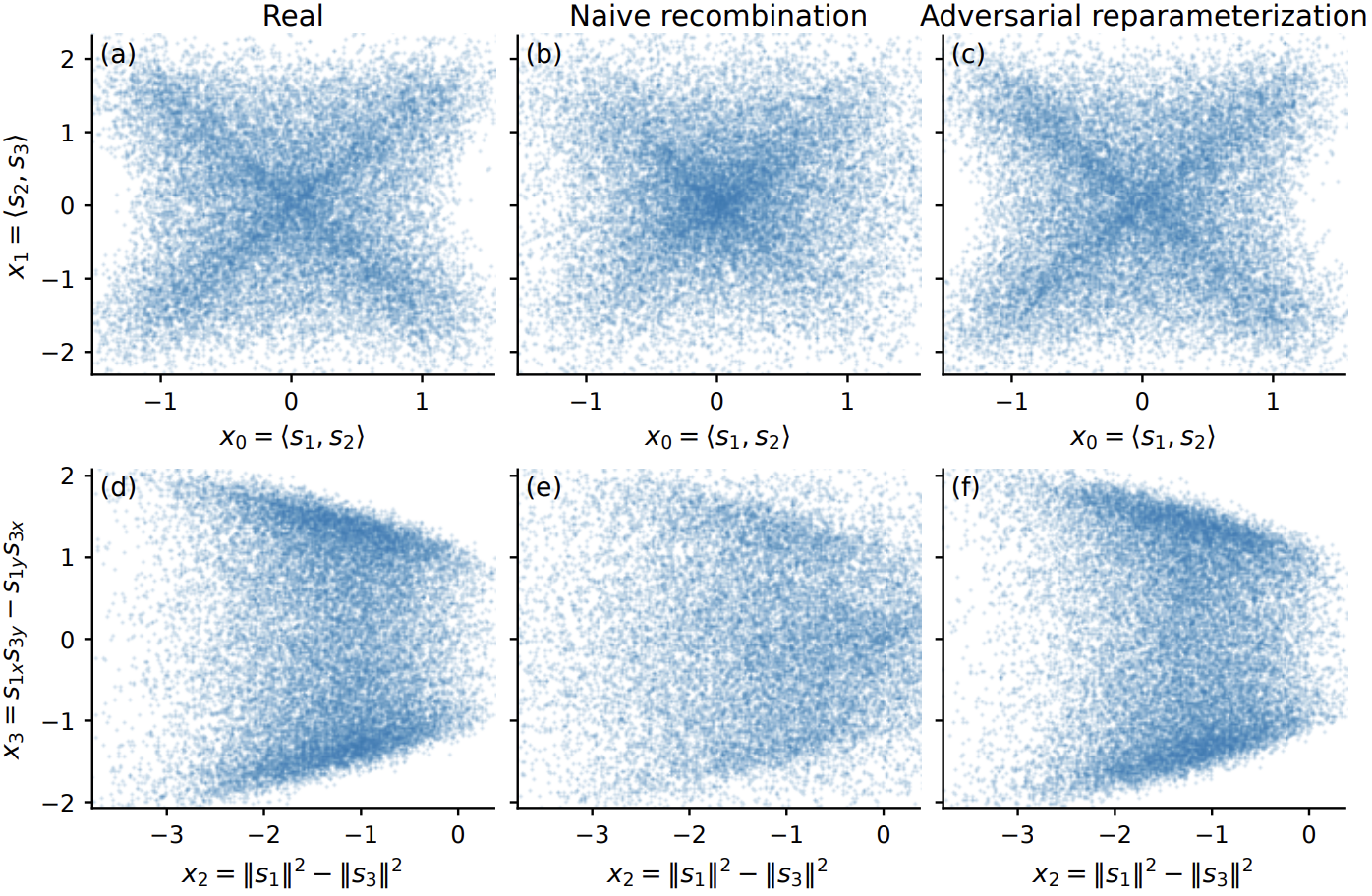}
  \caption{
  Synthetic recombination experiment.
  Left: real samples.
  Middle: naive recombination in entangled coordinates produces off-manifold
  samples.
  Right: discriminator feedback restores the geometry of real data.
  Top row shows $(x_0, x_1)$ projections; bottom row shows $(x_2, x_3)$ projections.
  }
  \label{fig:synthetic_scatter}
\end{figure}

\begin{table}[h!]
\centering
\caption{Synthetic recombination metrics (mean $\pm$ std over 3 seeds).
Lower is better ($\downarrow$), higher is better ($\uparrow$).
For reference, real samples have Mahalanobis$^2 = 3.99 \pm 0.002$.}
\label{tab:synthetic_metrics}
\begin{tabular}{l c c}
\toprule
Metric & Naive recomb. & + discriminator \\
\midrule
Mahalanobis$^2 \downarrow$
  & $7.637 \pm 0.459$
  & $4.101 \pm 0.133$ \\
MCC $\uparrow$
  & $0.603 \pm 0.018$
  & $0.757 \pm 0.018$ \\
MIG $\uparrow$
  & $0.060 \pm 0.008$
  & $0.122 \pm 0.043$ \\
TC $\downarrow$ & $1.0946 \pm 0.0405$ & $0.8409 \pm 0.3163$ \\
Block corr $\downarrow$
  & $0.214 \pm 0.007$
  & $0.039 \pm 0.060$ \\
\bottomrule
\end{tabular}
\end{table}

\begin{figure}[hbt!]
    \includegraphics[width=\columnwidth]{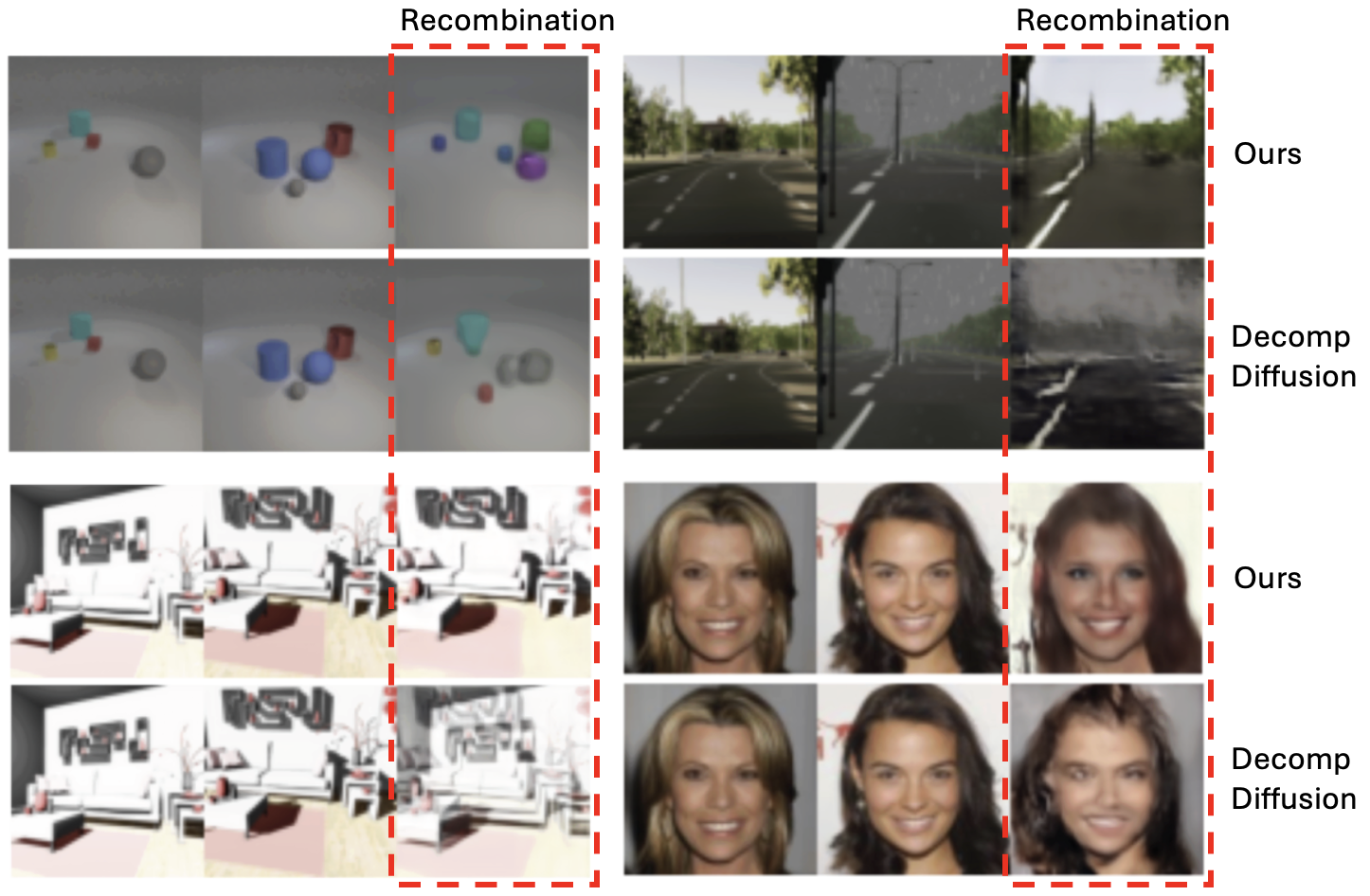}
    \caption{For each subfigure, the first two images are source samples drawn from the dataset, and the third image is a recombined generation obtained by mixing latent factors from the two sources.
Across all datasets (CLEVR, Virtual KITTI, Falcor3D, and CelebA-HQ; left-to-right, top-to-bottom), our method produces recombined samples that remain visually coherent and structurally consistent, while Decomp Diffusion often exhibits recombination artifacts.\looseness=-1}
    \label{fig:img_comps}
\end{figure}

\begin{figure*}[t]
    \centering
\includegraphics[width=0.9\textwidth]{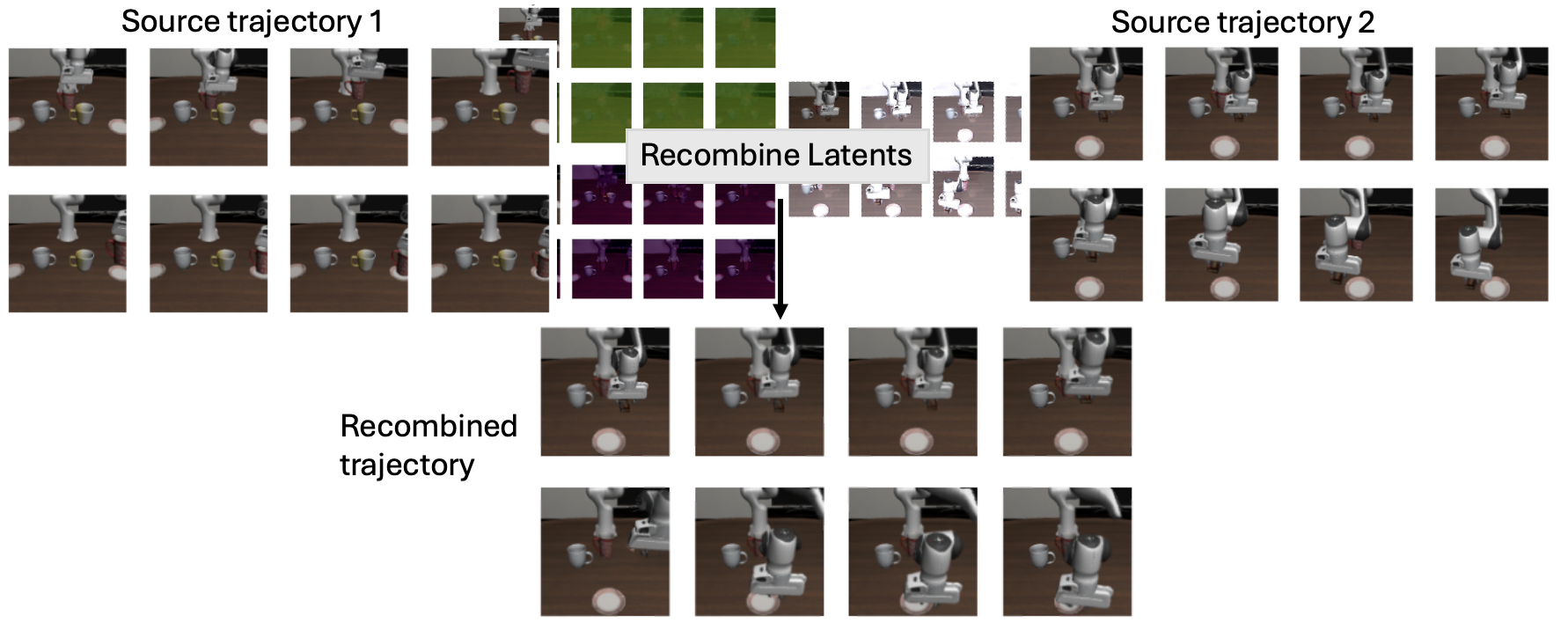} 
    \caption{Recombined video plan from Scene 5 and Scene 6 from Libero. The two source scenes are ``Red mug on right plate" from Scene 5 and ``Chocolate pudding left of plate" from Scene 6. The generated video plan shows the robot attempting to place the pudding on the edge of the plate. This specific action sequence does not appear in the original demonstrations, which only include placements beside the plate. The influence of Scene 5 can be observed in the trajectory as the robot arm is pulled sideways, suggesting interactions between the recombined behaviors.\looseness=-1}
\label{fig:recombine}
\end{figure*}

\subsection{Image Settings}

To evaluate the effectiveness of our proposed approach in image recombination, we compare our discriminator-driven method against the baseline Decomp Diffusion model across the image datasets Falcor3D \cite{falcor3d}, CelebA-HQ \cite{celebahq}, VKITTI \cite{7780839}, and CLEVR \cite{clevr}. Our evaluation primarily focuses on the image quality of recombined samples. Then, we discuss the empirical findings of image quality and disentanglement using the Falcor3D dataset. We compare against a reimplementation of Decomp Diffusion under our evaluation pipeline. Absolute metrics may differ from previously reported values; we focus on relative improvements under matched conditions in our codebase.\looseness=-1

\paragraph{Perceptual Quality of Recombined Images.} 
We assess the perceptual quality of recombined images through qualitative inspection and the Fréchet Inception Distance (FID). Table \ref{table:fid} presents the FID scores for both methods across datasets. Our approach consistently outperforms the baseline, reducing FID scores across all datasets. Our results highlight how a discriminator feedback can improve the perceptual consistency as well as finding more independent and semantically meaningful latent components. \looseness=-1

\paragraph{Disentanglement Evaluation. } 
We additionally evaluate how well our approach promotes disentangled latent spaces. We use MIG and MCC as metrics. We evaluate on the Falcor3D dataset, a synthetic benchmark for evaluating disentangled representations in generative models. With ground truth factor labels, Falcor3D enables precise assessment of whether models successfully separate independent attributes and serves as a benchmark for disentanglement metrics. Table \ref{table:disentanglement} shows that our method improves both MIG and MCC scores for the Falcor3D dataset. For a fixed trained model, repeated evaluation yields consistent MIG and MCC scores. These results suggest the effect of the discriminator can lead to the discovery of better factorization of independent generative factors.\looseness=-1

\begin{table}[htbp]
\centering
\label{tab:fid_best}
\begin{tabular}{lcc}
\toprule
\textbf{Dataset} & \textbf{Decomp Diffusion} & \textbf{Ours} \\
\midrule
CelebA-HQ  & $82.70\pm 19.39$ & $43.98\pm 0.78$ \\
Falcor3D  & $157.11\pm 9.53$ & $130.19\pm 3.62$ \\
VKITTI & $88.46\pm3.81$ & $84.22\pm 0.19$ \\
CLEVR & $25.70\pm 3.75$ & $ 24.16\pm 0.16$ \\
\bottomrule
\end{tabular}
\caption{\textbf{FID comparison on \textit{recombined samples}.} 
Samples from Decomp Diffusion and from training with discriminator feedback. For each dataset, we generated 10K samples from recombined latents. We compute the FID  (lower is better) against the original dataset distribution, and use it as a proxy for perceptual quality.}
\label{table:fid}
\end{table}
\begin{table}[hbt]
\centering
\begin{tabular}{lcccccccc}
\toprule
\textbf{Model} 
 & \textbf{MIG} $\uparrow$ & \textbf{MCC} $\uparrow$ \\
\midrule
InfoGAN & $0.025 \pm 0.01$ & $0.527 \pm 0.019$ \\
$\beta$-VAE & $0.090 \pm 0.01$ & $0.517 \pm 0.016$ \\
MONet & $0.139 \pm 0.020$ & $0.657 \pm 0.008$ \\
COMET* & $0.010 \pm 0.002$ & $0.573 \pm 0.003$ \\
Decomp Diffusion* & $0.065 \pm 0.019$ & $0.657 \pm 0.025$ \\
Ours* ($\lambda = 0.005$) & $0.118 \pm 0.012$ & $0.707 \pm 0.011$ \\
Ours* ($\lambda = 0.01$) & $0.116 \pm 0.058$ & $0.685 \pm 0.020$ \\
Ours* ($\lambda = 0.02$) & $0.141 \pm  0.031$ & $0.640 \pm 0.038$\\
\bottomrule
\end{tabular}
\caption{Comparison of MIG and MCC disentanglement metrics on Falcor3D. Asterisks denote models reimplemented and evaluated in our codebase (mean ± std over three independent training runs); unasterisked entries report previously published results and are included for reference. Brief summaries of baseline methods are provided in the appendix.}
\label{table:disentanglement}
\end{table}
Our approach optimizes both disentanglement and the compositional generation process itself, by regularizing recombined samples to be indistinguishable from single-source generations.
This contrasts with prior work that studies disentanglement or compositional generalization separately, and aligns with findings that disentangled representations alone may still fail to compose correctly \citep{montero2022lost, duan2020}.\looseness=-1

\subsection{Videos for Robotics Setting}
\label{sec:videos-robotics}
Robotic demonstration datasets are typically narrow: they contain a limited set of task instances, object placements, and action styles. Yet, downstream planning and control require coverage: a policy must succeed across many initial states and subtle variations that were never demonstrated. A natural way to obtain additional training signal is to generate new plausible rollouts. Many tabletop manipulation tasks are approximately compositional in time: a trajectory can be viewed as a sequence of reusable action primitives (reach, grasp, lift, transport, place) modulated by scene context and object identity. If a video model learns latent factors that isolate these primitives and their relevant object interactions, then recombining latents across demonstrations can synthesize task-relevant rollouts that are novel yet physically plausible. We demonstrate that unsupervised components of a video diffusion model can be used for \emph{directed exploration} in robotic environments by generating and executing recombined rollouts that expand state-space coverage beyond demonstrations.
\looseness=-1

We evaluate our approach in LIBERO \cite{libero}, a simulated environment designed for tabletop manipulation using a Franka Panda robot, featuring a variety of dexterous manipulation tasks. Each task requires the agent to reach a specified final state, which is defined by a descriptive sentence indicating the target object and the desired outcome. The action space includes changes in the end effector’s position and orientation, as well as the gripper’s applied force, resulting in a 7-dimensional control space. \looseness=-1

\textbf{Problem setup.}
We consider a finite-horizon partially observed control setting with observation space $\mathcal{O}\!\subset\! \mathbb{R}^{H \times W \times C}$ and action space $\mathcal{A}\!\subset\! \mathbb{R}^{7}$ (3D translation, 3D rotation, gripper open/close). A video of length $N$ is a sequence $x_{1:N} = (x_1,\dots,x_N)$ with $x_n\!\in\!\mathcal{O}$. We assume access to a dataset of demonstration videos $\mathcal{D}\!=\!\{x^{(m)}_{1:N_m}\}_{m=1}^M$ recorded in a table-top manipulation benchmark. We evaluate on two manipulation scenes from the LIBERO benchmark, denoted Scene 5 and Scene 6. Scene 5 involves placing mugs of varying appearance onto designated plates (e.g., “put red mug on left plate”), while Scene 6 focuses on placing different objects relative to a plate (e.g., “put chocolate pudding to the right of plate”). Each scene contains multiple task variants that differ in object type, color, and target relation; the complete list of task specifications is provided in Appendix \ref{appendix:libero}.\looseness=-1

\paragraph{Factorized video representation.}
Let $K$ denote the number of latent factors. An encoder $E_{\phi}$ maps a video suffix to $K$ latent components $z_{1:K}= E_{\phi}\big(x_{2:N}\big)$ where  $z_k \in \mathbb{R}^{d_k}$. Generation is performed by a diffusion model that is conditioned on the first frame. The denoiser $\bar{G}_{\theta}$ predicts the original video at every time-diffusion step $t$ conditioned on (i) the noised video, (ii) the factor latents, and (iii) the initial frame
$\epsilon_{\theta}\big(\tilde{x}_{1:N}^{(t)}, t \,\mid\, x_1, z_{1:K}\big)$. First-frame conditioning stabilizes training and focuses the representation on action primitives and object interactions rather than the initial scene layout.\looseness=-1

\textbf{Training objective with discriminator feedback.}
The base diffusion objective uses an $\ell_2$ reconstruction loss on the predicted clean sample:
\[
\mathcal{L}_{\mathrm{diff}}(\theta,\phi)
\!=\!\mathbb{E}\!\!_{\substack{x_{1:N}\sim \mathcal{D},\\
t\sim \mathcal{U}\{1,\dots,T\},\\
\epsilon\sim \mathcal{N}(0,I)}}\!
\!\Big[
\big\|
x_{1:N}
-\hat{x}_{\theta}\big(\tilde{x}^{(t)}_{1:N}, t \big| x_1, z_{1:K}\big)
\big\|_2^2
\Big],
\]
where $\tilde{x}^{(t)}_{1:N}$ denotes the noised trajectory at diffusion step $t$.

Given source videos $x^A$ and $x^B$ and timestep $t$, we introduce a discriminator $D_{\psi}$ that distinguishes single-source generations
$\hat{x}^{\text{single}}_{1:N} = \bar{G}_{\theta}(x^{A,(t)},t,x^A_1, z_{1:K})$
from recombined generations
$\hat{x}^{\text{recomb}}_{1:N} = \bar{G}_{\theta}(x^{A,(t)}, t,x^A_1,\tilde{z}_{1:K})$ where $\bar{G}_{\theta}(x,t,x_1,z_{1:K})=\frac{1}{K}\sum_{j=1}^K \bar{G}_{\theta}(x,t,x_1,z_j)$.
The discriminator is trained with the objective
\begin{align*}
\mathcal{L}_{D}(\psi)
\!&=\!\!-\mathbb{E}\big[\log D_{\psi}(\hat{x}^{\text{single}}_{1:N})\big]\!-\!\mathbb{E}\big[\log\big(1\!-\!D_{\psi}(\hat{x}^{\text{recomb}}_{1:N})\big)\big],
\end{align*}
while the generator and encoder are trained adversarially to make recombined generations
indistinguishable from single-source ones:
\begin{align*}
\mathcal{L}_{\text{adv}}(\theta,\phi)
&=
- \mathbb{E}\big[\log D_{\psi}(\hat{x}^{\,\text{recomb}}_{1:N})\big].
\end{align*}
The generator and encoder are optimized with the combined objective
$\mathcal{L}_{\text{diff}} + \lambda\,\mathcal{L}_{\text{adv}}$,
while the discriminator is optimized with $\mathcal{L}_{D}$.
Training alternates between these objectives. Intuitively, $D_{\psi}$ provides a self-supervised signal that penalizes recombinations
that are less consistent with the distribution of internally generated single-source rollouts.\looseness=-1


\begin{algorithm}[t]
\caption{Compositional video training with single-source vs recombined discriminator}
\label{alg:video-train}
\begin{algorithmic}[1]
\STATE \textbf{Input:} dataset $\mathcal{D}$, encoder $E_{\phi}$, denoiser $\epsilon_{\theta}$, discriminator $D_{\psi}$, weight $\lambda$
\WHILE{not converged}
  \STATE Sample two trajectories $x^{A}_{1:N}, x^{B}_{1:N} \sim \mathcal{D}$ and a recombination mask $S \in \{0,1\}^{K}$
  \STATE Encode latents $z^{A}_{1:K}\!\leftarrow\! E_{\phi}(x^{A}_{2:N}), z^{B}_{1:K}\!\leftarrow\!E_{\phi}(x^{B}_{2:N})$
  \STATE Recombine latents: $\tilde{z}_{k} \leftarrow S_k \, z^{A}_k + (1-S_k) \,z^{B}_k$ for $k=1,\dots,K$
  \STATE Compute diffusion loss $\mathcal{L}_{\text{diff}}$ on $(x^{A}_{1:N})$  \COMMENT{standard diffusion training}
  \STATE Sample $t \sim \mathrm{Unif}(\{1,\dots,T\})$, $\epsilon\sim\mathcal{N}(0,I)$
  \STATE Forward diffusion: $x^{A,(t)} \leftarrow \sqrt{\bar{\alpha}_t}\,x^{A}_{1:N} + \sqrt{1-\bar{\alpha}_t}\,\epsilon$

  \STATE \textbf{Single-source:} 
  $\hat{x}^{\text{single}} \!\leftarrow\!\bar{G}_{\theta}(x^{A,(t)},t,x^A_1,z^{A}_{1:K})$ by conditioning on $(x^{A}_1, z^{A}_{1:K})$
  \STATE \textbf{Recombined:} 
  $\hat{x}^{\text{recomb}} \!\leftarrow\!\bar{G}_{\theta}(x^{A,(t)},t,x^A_1,\tilde{z}_{1:K})$
  conditioning on $(x^{A}_1, \tilde{z}_{1:K})$
  \STATE \textbf{Discriminator update (single-source/recombined):}\looseness=-1 
  \STATE
$\mathcal{L}(\psi)
\!=\! -\mathbb{E}\!\left[
\log D_\psi(\hat{x}^{\text{single}}_{1:N})
\!+\! \log\!\left(1 \!-\! D_\psi(\hat{x}^{\text{recomb}}_{1:N})\right)
\right]$
  \STATE $\psi \leftarrow \psi - \eta_{\psi}\nabla_{\psi}\mathcal{L}$
  \STATE \textbf{Generator/encoder update:}
  \STATE \hspace{1em} $\mathcal{L}_{\text{adv}}(\theta,\phi) \;=\; -\mathbb{E}\!\left[\log D_{\psi}(\hat{x}^{\,\text{recomb}}_{1:N})\right]$
  \STATE \hspace{1em} $\theta,\phi \leftarrow \theta,\phi - \eta\nabla_{\theta,\phi}\left(\mathcal{L}_{\text{diff}} + \lambda\mathcal{L}_{\text{adv}}\right)$
\ENDWHILE
\end{algorithmic}
\end{algorithm}

\textbf{Evaluation protocol and baselines.}
We evaluate generated rollouts using two complementary metrics: a \emph{physical consistency} score $L_{\text{phys}}$, which measures disagreement between actions inferred from generated frames and those executed in the environment, and a \emph{coverage} metric $\mathrm{Cov}_{\Delta}$, which quantifies exploration breadth in discretized joint-state space.
Formal definitions and implementation details are provided in Appendix~\ref{appendix:libero}. To compare exploration fairly across planning strategies, we adopt a paired-episode evaluation protocol. For each episode, we fix the task instance, demonstration index, environment seed, and a short warm-up horizon during which ground-truth actions from the demonstration are executed. We include a short warm-up phase so that the robot arm reliably picks up the object before executing generated plans. Without this warm-up, recombined trajectories frequently fail to grasp the object, resulting in rollouts that terminate early or do not meaningfully explore the scene. This reflects a current limitation of our approach, which focuses on recombining and executing trajectories after object pickup. Addressing reliable grasping jointly with compositional planning is an important direction for future work. 

All methods are evaluated from the identical post--warm-up state using the same pretrained visuomotor diffusion policy and execution horizon; they differ only in how the goal video is constructed. Our \emph{Recombine+Generate} method recombines latent factors across two demonstrations and decodes them with the diffusion model to produce a novel goal video. As a non-generative baseline, \emph{Demo-Self} uses the remaining frames of the same demonstration as the goal video, without recombination or generation. The \emph{Demo-Other} baseline instead uses a goal video taken from a different demonstration of the same task. Finally, \emph{Shuffle-Self} forms the goal video by locally shuffling frames within short temporal windows of the same demonstration, preserving coarse temporal drift while randomizing local temporal structure. Each method is run for 10 paired episodes, and we report the mean $\pm$ standard deviation of joint-state coverage as a function of execution steps after warm-up.
(Fig.~\ref{fig:states_over_time}).\looseness=-1

\begin{figure}[t]
  \centering
  \includegraphics[width=0.85\linewidth]{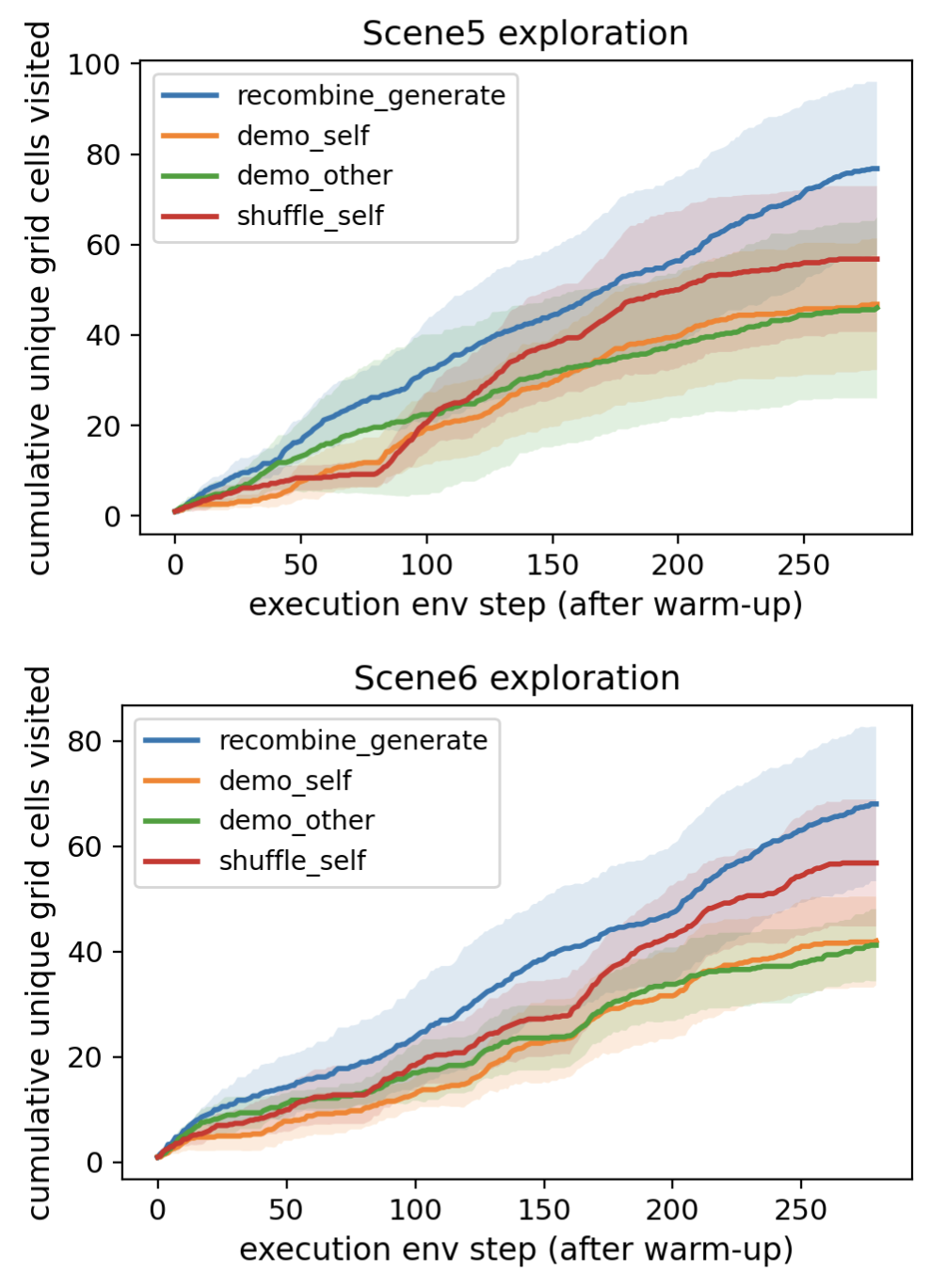}
  \caption{
 Mean ± std joint-state coverage over 10 paired runs for Scene5 (top) and Scene6 (bottom). Recombined videos yield action trajectories with consistently higher exploration than non-compositional baselines under identical warm-up conditions.
  }
\label{fig:states_over_time}
\end{figure}
\textbf{Qualitative results.}
Figure \ref{fig:recombine} is a demonstration of producing a new robotic video sample based on combining latent factors from two different videos in the training dataset.

\textbf{Ablations and trade-offs.} We sweep $\lambda$ and observe an inverted U-shaped behavior: small values reduce $\mathcal{L}_{\text{phys}}$ and improve $\mathrm{Cov}_{\Delta}$ by pruning unlikely recombinations, whereas excessively large values over-regularize, collapsing factor semantics and harming reconstruction. This mirrors the image setting but is now measured with action-grounded metrics. The video formulation casts recombination as factor swapping in a first-frame conditioned diffusion model, regularized by a sequence-level discriminator. Physical plausibility is evaluated by rolling out actions extracted from generated plans, while coverage captures the combinatorial reach of recomposed behaviors. This method yields a principled, action-grounded way to assess both realism and exploratory generalization in robotics videos. We are not aware of prior work that evaluates latent recombination in generative models through action-grounded metrics that assess physical executability and exploration coverage.

Table \ref{table:state-space} reports the state-space coverage across different models. We observe a consistent increase in explored grid cells when incorporating the discriminator, where the best-performing variant achieves improvements of 35.8\% in Scene5 (9440 → 12816) and 73.8\% in Scene6 (5400 → 9385) over the baseline. \looseness=-1

\section*{Acknowledgements}
We gratefully acknowledge insightful discussions with Zhuo Chen, Ajay Arora, Sarah Liaw, and Anastasiia Alokhina. Research was sponsored by the Department of the Air Force Artificial Intelligence Accelerator and was accomplished under Cooperative Agreement Number FA8750-19-2-1000. The views and conclusions contained in this document are those of the authors and should not be interpreted as representing the official policies, either expressed or implied, of the Department of the Air Force or the U.S. Government. The U.S. Government is authorized to reproduce and distribute reprints for Government purposes notwithstanding any copyright notation herein. In addition, this work is supported by the National Science Foundation under Cooperative Agreement PHY-2019786 (The NSF AI Institute for Artificial Intelligence and Fundamental Interactions, http://iaifi.org/). Archer Wang is supported by the NSF graduate research fellowship. Finally, this work was supported by NSF grant CCF 2338816, and has been made possible in part by a gift from the Chan Zuckerberg Initiative Foundation to establish the Kempner Institute for the Study of Natural and Artificial Intelligence at Harvard University.\looseness=-1

\section*{Impact Statement}
Our proposed approach poses no immediate social risks in its current form, as it is primarily evaluated on standard datasets and simulated environments. However, the ability to manipulate latent representations could be misused in ways that amplify biases inherent in the training data, leading to unintended consequences in decision-making systems. Ensuring responsible deployment is crucial to avoid potential ethical risks. Despite these concerns, our approach offers significant potential benefits across various domains, including robotics, automated planning, and creative content generation, by enabling improved generalization and compositional understanding of complex environments.

\nocite{dai2024advdiffgeneratingunrestrictedadversarial}
\bibliography{main}
\bibliographystyle{icml2026}

\newpage
\appendix
\onecolumn

\section{Detailed Background on Decomp Diffusion and COMET}
\label{app:decomp_comet_background}

This appendix reviews two closely related frameworks for unsupervised compositional representation learning that our work builds on top of: COMET \cite{du2021unsupervised} and Decomp Diffusion \cite{decompdiffusion}. Both methods aim to decompose images into independently manipulable latent factors and to support recombination by composing these factors at generation time. The primary difference lies in how composition is implemented and optimized.

\subsection{COMET: Composable Energy-Based Representations}

COMET represents each latent factor as an energy function over images. Given an image $x_i \in \mathbb{R}^D$, an encoder produces a set of latent variables $z_1, \dots, z_K = \mathrm{Enc}_\phi(x_i)$, and a shared energy network defines a family of energies $E_\theta(x, z_k) : \mathbb{R}^D \times \mathbb{R}^M \to \mathbb{R}$. A single latent $z_k$ specifies an energy landscape whose minimum corresponds to an image consistent with that factor. Composition is defined by summing energies and jointly minimizing over the image:
\begin{equation}
x^\star(z_{1:K})
=
\arg\min_x \sum_{k=1}^K E_\theta(x, z_k).
\label{eq:comet_compose}
\end{equation}
COMET is trained by reconstructing the input image as the minimum-energy configuration under the inferred latents. The learning objective minimizes the reconstruction error between the input and the result of energy minimization,
\begin{equation}
\mathcal{L}_{\mathrm{COMET}}
=
\left\|
\arg\min_x \sum_{k=1}^K E_\theta\!\big(x, \mathrm{Enc}_\phi(x_i)[k]\big)
-
x_i
\right\|_2^2.
\label{eq:comet_obj}
\end{equation}
Since the argmin is intractable to compute, it is approximated by unrolling several steps of gradient descent on $x$,
\begin{equation}
x^{(n)} = x^{(n-1)} - \lambda \nabla_x \sum_{k=1}^K E_\theta(x^{(n-1)}, z_k),
\qquad n=1,\dots,N,
\label{eq:comet_unroll}
\end{equation}
with $x^{(0)}$ initialized from noise. Training therefore requires backpropagation through an unrolled optimization procedure, which can be computationally expensive and sensitive to hyperparameters. Despite this cost, COMET provides a conceptually clean formulation of compositionality: factors are combined by additive energies, and recombination is implemented by minimizing the same summed objective using latents drawn from different images.

\subsection{Decomp Diffusion: Composing Denoising Components}

Decomp Diffusion builds on diffusion probabilistic models and replaces explicit energy minimization with diffusion denoising dynamics. For the diffusion training in Decomp Diffusion \cite{decompdiffusion}, a noisy image
\[
x_i^{t} = \sqrt{\bar\alpha_t}\,x_i + \sqrt{1-\bar\alpha_t}\,\epsilon,
\qquad \epsilon \sim \mathcal{N}(0,I),
\]
is passed to a denoiser $\epsilon_\theta(x_i^t, t)$ trained to predict the noise (or original image) via
\begin{equation}
\mathcal{L}_{\mathrm{DDPM}}
=
\left\|
\epsilon - \epsilon_\theta(x_i^t, t)
\right\|_2^2.
\label{eq:ddpm_loss}
\end{equation}


When the denoiser is conditioned on a latent $z$, $\epsilon_\theta(x_t, t, z)$ corresponds to the gradient of a latent-dependent energy $E_\theta(x, z)$. To enable decomposition, Decomp Diffusion encodes each input image into $K$ latents,
\[
z_1,\dots,z_K = \mathrm{Enc}_\phi(x_i),
\]
and trains a shared conditional denoiser such that the mean of per-latent denoising predictions matches the true noise (or original image):
\begin{equation}
\frac{1}{K}\sum_{k=1}^K \epsilon_\theta(x_i^t, t, z_k) \approx \epsilon.
\label{eq:decomp_sum}
\end{equation}
The resulting training objective is
\begin{equation}
\mathcal{L}_{\mathrm{Decomp}}
=
\left\|
\epsilon - \frac{1}{K}\sum_{k=1}^K \epsilon_\theta(x_i^t, t, z_k)
\right\|_2^2.
\label{eq:decomp_loss}
\end{equation}
Latent dimensionality constraints act as an information bottleneck, encouraging different $z_k$ to capture distinct factors. At inference time, composition is performed by averaging denoising predictions from any chosen set of latents at each diffusion step, $\epsilon_{\mathrm{pred}}(x_t, t) = \frac{1}{K}\sum_{k} \epsilon_\theta(x_t, t, z_k),
$ and running the standard diffusion sampling procedure using $\epsilon_{\mathrm{pred}}$. This enables recombination without solving an inner optimization problem.

\subsection{Conceptual Relationship Between COMET and Decomp Diffusion}

Both COMET and Decomp Diffusion define composition through additive structure. COMET composes factors by summing energies and explicitly minimizing over the image, while Decomp Diffusion composes factors by averaging denoising (energy-gradient) contributions and relying on diffusion dynamics to perform implicit optimization. The latter avoids backpropagation through unrolled image updates during training and yields a more stable and scalable learning procedure, while preserving the same compositional intuition.

\subsection{Other Baseline Methods} InfoGAN \citep{infogan} extends GANs by maximizing mutual information between a subset of latent variables and generated samples, encouraging individual latent dimensions to capture interpretable factors of variation. $\beta$-VAE \cite{betavae} promotes disentangled representations by upweighting the KL divergence term in the variational objective, encouraging statistical independence among latent dimensions within a fixed-dimensional latent space. MONet \cite{burgess2019monet} is an object-centric generative model that decomposes scenes into multiple components using an iterative attention mechanism and soft segmentation masks, enabling pixel-level decomposition into object-like regions while primarily focusing on spatial structure rather than global scene factors.



\section{Experiment Details}
In this section, we provide dataset details in \cref{appendix:dataset_details}, and present training and inference
details of our method in \cref{appendix training details} and  \cref{appendix inference details}.

\subsection{Dataset Details}
\label{appendix:dataset_details}

We evaluate our method on four image datasets commonly used for studying compositionality and disentangled representations: CelebA-HQ, Falcor3D, Virtual KITTI 2, and CLEVR.

\paragraph{CelebA-HQ.}
CelebA-HQ~\cite{celebahq} is a high-quality face image dataset derived from CelebA, consisting of aligned human face images with rich variation in global attributes such as facial structure, hair style, hair color, skin tone, and expression. Following prior work, we use images resized to $64 \times 64$ resolution.

\paragraph{Falcor3D.}
Falcor3D~\cite{falcor3d} is a synthetic dataset rendered using a physically based graphics pipeline, designed specifically for evaluating disentanglement. Images depict simple 3D scenes with controlled generative factors including object geometry, material properties, lighting direction and intensity, and camera pose. Ground-truth factor annotations are available, enabling quantitative evaluation of disentanglement metrics such as MIG and MCC. We use the standard $128 \times 128$ image resolution.

\paragraph{Virtual KITTI 2.}
Virtual KITTI 2 (VKITTI2)~\cite{7780839} is a photorealistic synthetic dataset that mirrors real-world driving scenes from the KITTI benchmark. It provides rendered outdoor scenes with structured variation in camera viewpoint, lighting, weather conditions, background layout, and object configurations. In our experiments, we use VKITTI2 to evaluate decomposition of complex global scene factors and the plausibility of recomposed scenes under latent recombination.

\paragraph{CLEVR.}
CLEVR~\cite{clevr} is a synthetic object-centric dataset consisting of images containing multiple simple 3D objects placed on a tabletop. Each scene varies along interpretable factors such as object shape, color, size, material, and spatial arrangement. Although originally designed for visual reasoning, CLEVR is widely used for evaluating object-level decomposition and compositional generalization. We use $64 \times 64$ rendered images and focus on recombination of object-centric latent factors.

\subsection{Training Overview}
\label{appendix training details}

We used standard denoising training to train our denoising
networks, with $1000$ diffusion steps and squared cosine beta
schedule. In Decomp Diffusion and in our implementation, the denoising network $\epsilon_\theta$ is trained to directly predict the original image $x_0$. To train our diffusion model that conditions on inferred
latents $z_k$, we first utilize the latent encoder to encode input images into features that are further split into a set of latent
representations $\{z_1,\dots, z_k\}$. For each input image, we
then train our model conditioned on each decomposed latent
factor $z_k$ using standard denoising loss. Each model is trained for 20k - 40k steps on an NVIDIA V100 32GB machine, and we use a batch size of 16 when training.


\paragraph{Single-step discriminator training.}
Applying the discriminator to fully sampled outputs would require running an
iterative denoising process (e.g., DDIM with tens of steps) during training,
significantly increasing computational cost and complicating gradient-based
optimization. Instead, we train the discriminator on denoised predictions obtained
from a single diffusion step while sharing the same noisy input $x_t$. Although these predictions are not exact samples from the model, they provide a
consistent and differentiable surrogate that allows the discriminator to detect
incoherent latent recombinations. Because DDIM sampling repeatedly applies the same denoiser across timesteps, shaping single-step predictions across noise levels is a practical proxy for improving the quality of the final multi-step samples.


\subsection{Inference Details}
\label{appendix inference details}

When generating images, we use DDIM with 50 steps for
faster image generation. 

\textbf{Decomposition.} To decompose an image $x$, we first
pass it into the latent encoder $\mathrm{Enc}_\phi$ to extract out latents $\{z_1, \cdots, z_K\}$. For each latent $z_k$, we generate an image
corresponding to that component by running the image generation algorithm on $z_k$.

\textbf{Reconstruction.} To reconstruct an image $x$ given latents $\{z_1, \cdots, z_K\}$, in the denoising process, we predict $\epsilon$ by
averaging the model outputs conditioned on each individual $z_k$. The final result is a denoised image which incorporates all inferred components, i.e., reconstructs the image.

\textbf{Recombination.} To recombine images $x$ and $x'$, we recombine their latents $\{z_1, \dots, z_K\}$ and $\{z_1', \cdots, z_K'\}$. We
select the desired latents from each image and condition on
them in the image generation process, i.e., predict $\epsilon$ in the denoising process by averaging the model outputs conditioned
on each individual latent.

To additively combine images $x$ and $x'$ so that the result
has all components from both images, e.g., combining two
images with 4 objects to generate an image with 8 objects,
we modify the generation procedure. In the denoising process, we assign the predicted $\epsilon$ to be the mean over all $2 \times K$ model outputs conditioned on individual latents in $\{z_1, \cdots, z_K\}$ and $\{z_1', \cdots, z_K'\}$. This results in an image
with all components from both input images.

\subsection{Ablations}
As \(\lambda\) increases, recombined sample quality initially improves on both datasets, 
indicating that discriminator feedback helps constrain generations toward the data manifold, as seen in Figure \ref{fig:gamma_fid_comparison}.
However, overly strong guidance leads to worse FID, suggesting that excessive discriminator pressure can conflict with the diffusion reconstruction objective. We also include a qualitative example in Figure \ref{fig:latents} highlighting a potential degenerate behavior of the adversarial objective when the specific discriminator weight is too high. This motivates using a moderate discriminator weight in all subsequent experiments.

\subsection{Metric computation (MIG, MCC).}
Let $r(x)\in\mathbb{R}^D$ denote the representation returned by our encoder and let $y$ denote the ground-truth generative factors. We construct $r(x)$ by concatenating the $K$ component vectors produced by the encoder and use the encoder mean for each component. We compute both MIG and MCC over the individual coordinates of $r(x)$, treating each coordinate as a scalar latent variable. We do not use a learned probe or regression for either metric.

\textbf{MIG.} We sample 4{,}000 observations and factors, compute representations, and discretize each latent coordinate independently using a histogram discretizer with 20 bins. We estimate mutual information using \texttt{sklearn.metrics.mutual\_info\_score}. For each factor, MIG is the difference between the largest and second-largest mutual informations across latent coordinates, normalized by the factor entropy. We report the mean and standard deviation across factors.

\textbf{MCC.} We compute Spearman correlations between each representation coordinate and each ground-truth factor, apply Hungarian matching to the absolute correlation matrix to assign coordinates to factors, and report the mean absolute matched correlation.

\begin{figure}[h!]
    \begin{center}
    \begin{subfigure}{0.4\columnwidth}
        \centering
        \includegraphics[width=\columnwidth]{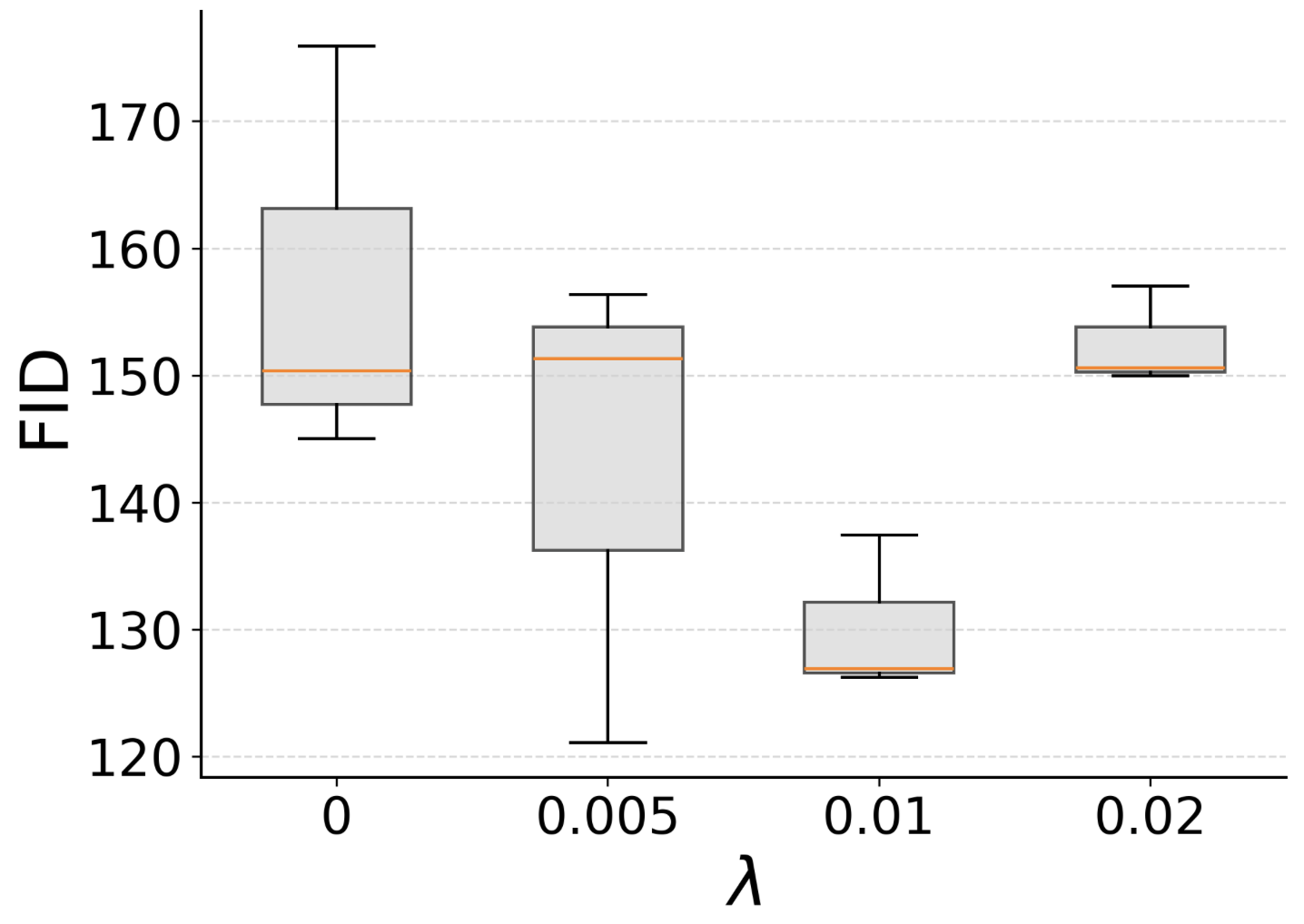} 
        \caption{Falcor3D.}
        \label{figs:fid_falcor3dv1}
    \end{subfigure}%
    \begin{subfigure}{0.4\columnwidth}
        \centering
        \includegraphics[width=\columnwidth]{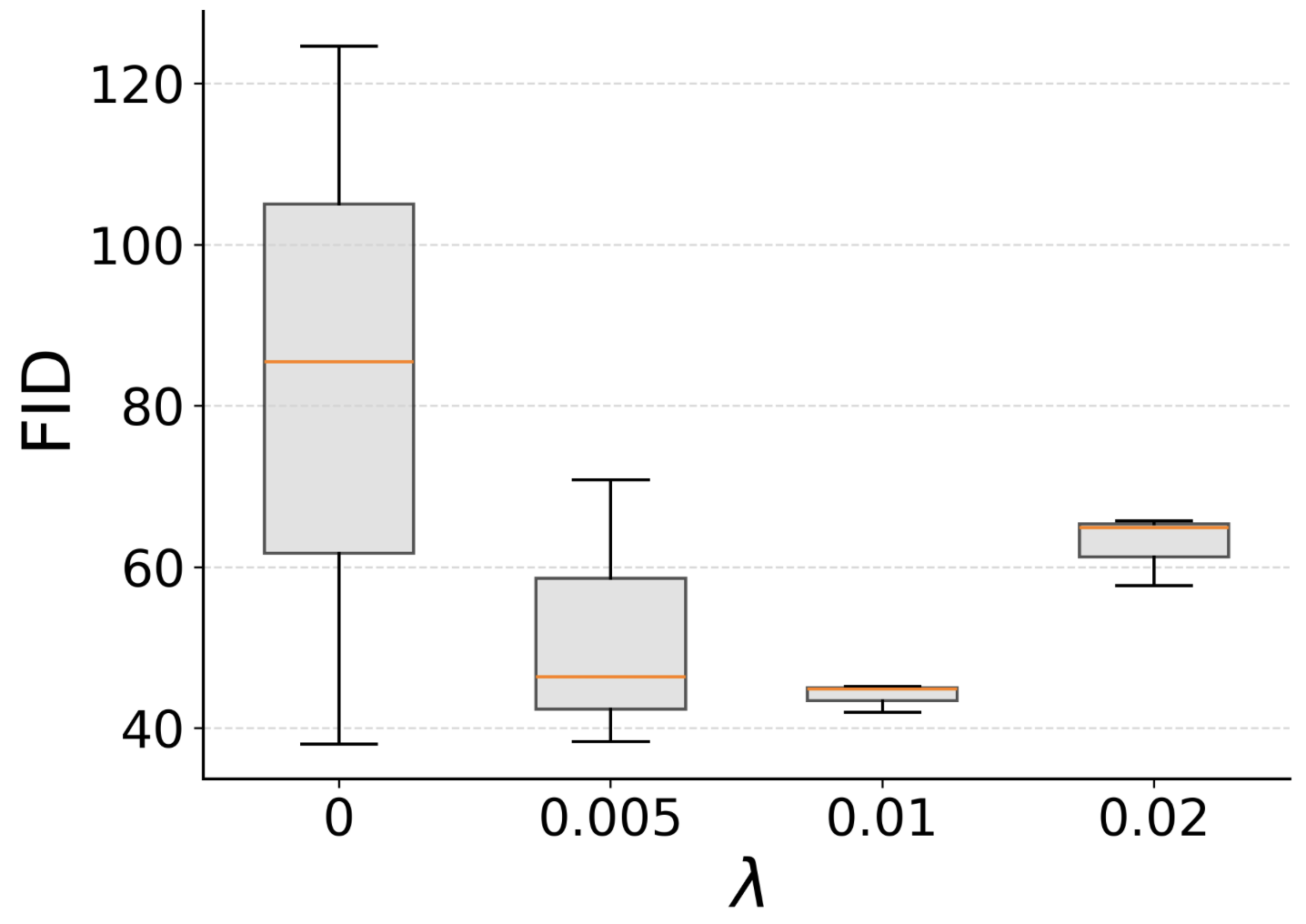} 
        \caption{CelebA-HQ.}
        \label{fig:fid_celebahq_2}
    \end{subfigure}
    \end{center}
    \caption{
        \textbf{Effect of discriminator strength (\(\lambda\)) on sample quality.}
        We plot the FID (lower is better) as a function of discriminator guidance weight \(\lambda\) for two datasets: 
        (a) \textbf{Falcor3D} and (b) \textbf{CelebA-HQ}. 
        Increasing \(\lambda\) initially improves FID, indicating that the discriminator effectively pulls recombined samples closer to the data manifold. 
        However, overly strong guidance (\(\lambda = 0.02\)) slightly degrades FID.
    }
    \label{fig:gamma_fid_comparison}
\end{figure}

\begin{figure}[h!]
    \centering
    \includegraphics[width=0.9\columnwidth]{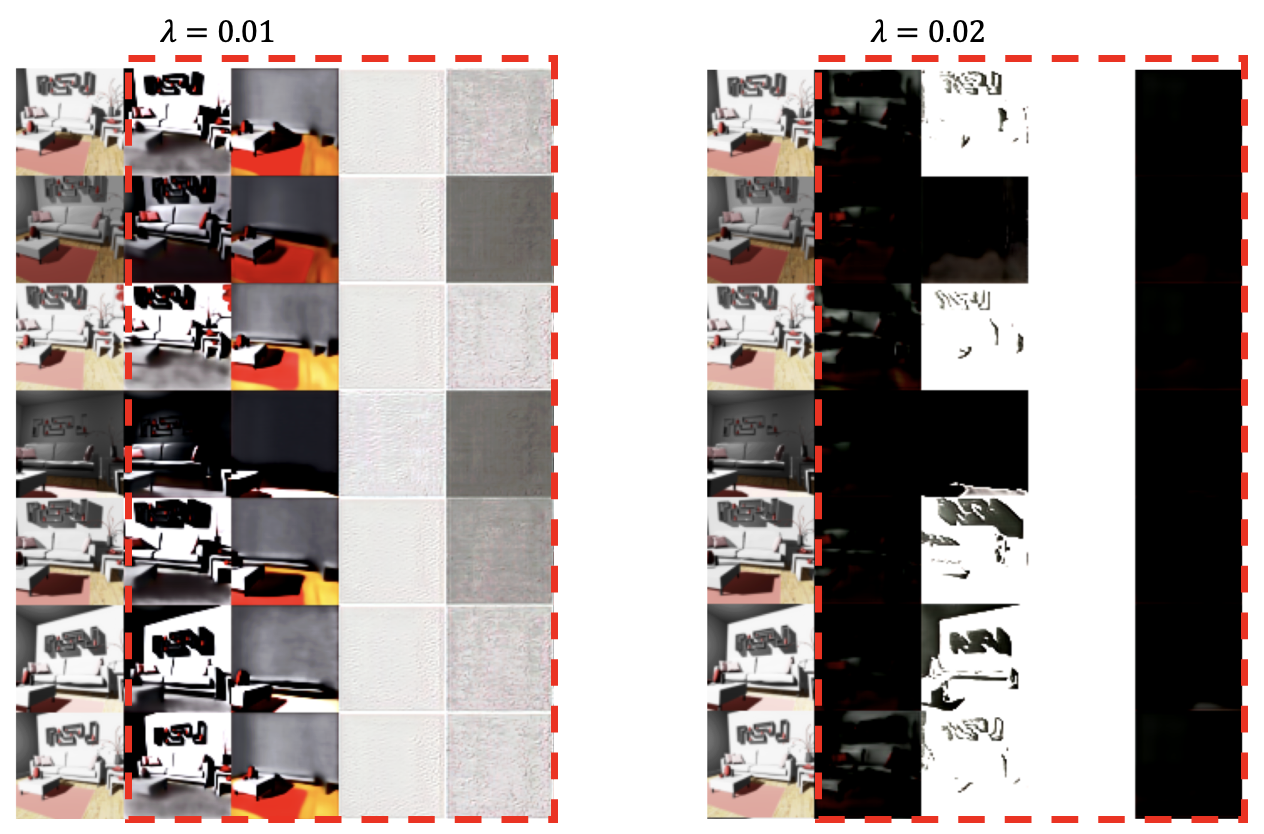}
    \caption{Qualitative visualization of learned latent components at two discriminator weights.
      \textbf{Left:} Non-collapsed representation at $\lambda=0.01$. \textbf{Right:} partial latent collapse observed in a single example at $\lambda = 0.02$. In each panel, the leftmost column shows input images from the dataset, while the red dashed region visualizes the outputs associated with individual latent components. Notably, for $\lambda = 0.02$, the last two latent columns produce nearly identical outputs across different inputs, indicating under-utilized or collapsed latent dimensions.
    }
    \label{fig:latents}
\end{figure}

\section{Additional Details: Synthetic Recombination Experiment}
\label{app:synthetic_recomb}

We construct independent latent factors
\( s = (s_1,s_2,s_3)\in\mathbb{R}^6 \),
with each \(s_k\in\mathbb{R}^2\) following a distinct non-Gaussian distribution
(ring, uniform square, and bimodal Gaussian).
Observed variables \(x\in\mathbb{R}^4\) are generated deterministically via
\begin{align}
x_0 &= \langle s_1, s_2 \rangle, \quad
x_1 = \langle s_2, s_3 \rangle, \\
x_2 &= \|s_1\|^2 - \|s_3\|^2, \quad
x_3 = s_{1x}s_{3y} - s_{1y}s_{3x},
\end{align}
which imposes explicit cross-factor constraints that characterize the data manifold.

To model entangled latent coordinates, factors are mapped through a fixed invertible linear transform
\(z = Ms\), where \(M\) is a random orthogonal matrix.
Naive recombination swaps 2D blocks directly in \(z\)-space between independent samples and then decodes using \(M^{-1}\).
Because blockwise recombination is not aligned with the true factorization,
this operation generically violates the interaction constraints above, producing off-manifold samples
(Fig.~\ref{fig:synthetic_scatter}, middle column).

We introduce a learnable linear reparameterization \( \hat{s} = W z \),
perform recombination in \(\hat{s}\)-coordinates, and map back via \(W^{-1}\) before decoding.
The matrix \(W\) is trained adversarially using a discriminator on \(x\),
which distinguishes real samples from recombined ones; a weak regularizer encourages \(W^\top W \approx I\).
This setup mirrors discriminator feedback used in diffusion-based recombination, while remaining analytically transparent. Although this does not strictly guarantee invertibility, 
$W$ does not become ill-conditioned in practice in our toy setting. We leave parameterizations with guaranteed invertibility (e.g., orthogonal or flow-based constructions) to future work.

To quantify deviation from the data manifold, we compute the squared Mahalanobis distance
\(
d_M(x) = (x-\mu)^\top \Sigma^{-1}(x-\mu)
\),
where \(\mu\) and \(\Sigma\) are the empirical mean and covariance of real samples in \(x\)-space.
This metric measures how inconsistent a sample is with the second-order statistics of real data,
independent of marginal scaling.
Discriminator feedback substantially reduces the mean Mahalanobis distance of recombined samples,
indicating recovery of a recombination-consistent coordinate system.\looseness=-1

\paragraph{Architecture and training details.}
The discriminator \(D:\mathbb{R}^4\to[0,1]\) is implemented as a lightweight multilayer perceptron
with two hidden layers of width 64, ReLU activations, and a sigmoid output.
It is trained using binary cross-entropy to distinguish real samples
\(x = f(s)\) from recombined samples obtained by blockwise recombination in latent space.
We emphasize that \(D\) has limited capacity and operates purely in observation space,
serving only to detect violations of the data manifold induced by recombination.

The reparameterization matrix \(W\in\mathbb{R}^{6\times6}\) is trained to minimize
\begin{equation}
\mathcal{L}(W)
\;=\;
\mathbb{E}_{\text{recomb}}\!\left[
-\log D\!\left(f\!\left(W^{-1} \,\mathrm{recombine}(Wz)\right)\right)
\right]
\;+\;
\lambda \,\|W^\top W - I\|_F^2,
\end{equation}
where the expectation is taken over randomly paired samples and random blockwise recombinations.
The decoder \(f\) and mixing matrix \(M\) are fixed throughout training.
No reconstruction loss is used, as the mapping from \(z\) to \(x\) is deterministic and invertible.

Recombination is performed by swapping entire two-dimensional blocks corresponding to latent components
between independently sampled latent vectors within a minibatch.
All models are trained using Adam with learning rate \(10^{-3}\), batch size 256,
and trained for 5{,}000 iterations.
Results are stable across random initializations of \(M\) and \(W\).

\subsection{Image Training Details} Hyperparameters and examples of recombined images are provided. \\
Baseline models have the same settings as the one from Decomp Diffusion.
For each dataset, we trained for 20k steps using a batch size of 16; for Falcor3D, we trained for 40k steps due to the higher 128×128 image resolution

Discriminator is adapted from DCGAN architecture. The discriminator is a convolutional neural network that maps an RGB image
$x \in \mathbb{R}^{3 \times H \times W}$ to a scalar in $(0,1)$.
For all experiments, images are resized to $H = W = 64$. The network consists of four strided convolutional blocks followed by a linear
discriminator:
\begin{itemize}
    \item \textbf{Conv1:} $4 \times 4$ convolution, $3 \rightarrow 64$ channels, stride $2$, padding $1$, followed by LeakyReLU$(0.2)$.
    \item \textbf{Conv2:} $4 \times 4$ convolution, $64 \rightarrow 128$ channels, stride $2$, padding $1$, followed by BatchNorm and LeakyReLU$(0.2)$.
    \item \textbf{Conv3:} $4 \times 4$ convolution, $128 \rightarrow 256$ channels, stride $2$, padding $1$, followed by BatchNorm and LeakyReLU$(0.2)$.
    \item \textbf{Conv4:} $4 \times 4$ convolution, $256 \rightarrow 512$ channels, stride $2$, padding $1$, followed by BatchNorm and LeakyReLU$(0.2)$.
\end{itemize}

Each convolution reduces the spatial resolution by a factor of $2$, yielding a
final feature map of size $512 \times (H/16) \times (W/16)$.
For $64 \times 64$ inputs, this corresponds to a $512 \times 4 \times 4$ tensor.
The final feature map is flattened and passed through a fully connected layer
of size $512 \cdot (H/16) \cdot (W/16) \rightarrow 1$, followed by a sigmoid
activation to produce a scalar output.

An interesting future direction is doing in-depth comparisons of how the model architectures change the effectiveness of the discriminator.

\subsection{Qualitative Results}
We show side-by-side comparisons between recombined samples produced by our method and those from Decomp Diffusion in Figure~\ref{fig:img_comps}. 
Overall, our method yields recombined images that are more physically plausible and visually coherent.
In contrast, Decomp Diffusion can exhibit off-manifold artifacts under recombination; for example, in \textbf{Falcor3D}, wall textures are inconsistently overlaid, suggesting that global scene factors such as layout and appearance are not coherently composed.

Figure \ref{fig:failed_disentanglement} illustrates a representative failure mode of Decomp Diffusion under latent recombination. While the baseline produces visually reasonable samples when generating from a single latent source, recombining latent factors from multiple images often leads to inconsistent global structure. In particular, in Falcor3D, camera and layout information are entangled across multiple latent dimensions, causing wall textures and scene geometry to be inconsistently overlaid when latents are mixed. This highlights the lack of an explicit mechanism in the baseline to enforce compositional consistency across global scene factors.



\begin{figure}[hbt!]
    \centering
    \begin{subfigure}{0.95\columnwidth}
        \centering
        \includegraphics[width=\linewidth]{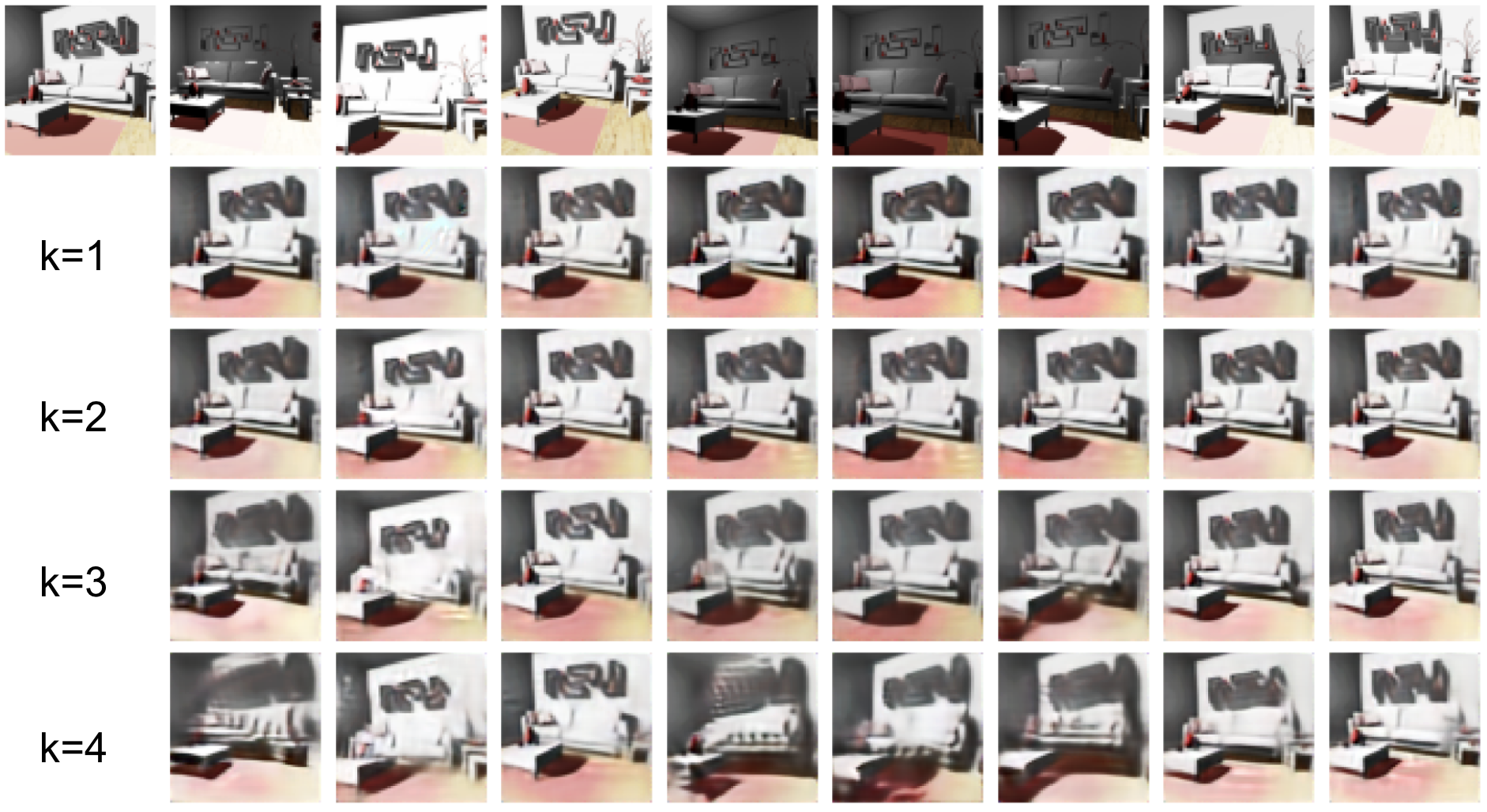}
        \caption{COMET}
        \label{fig:fid_falcor3dv3}
    \end{subfigure}
    \par\vspace{2.0em}
    \begin{subfigure}{0.95\columnwidth}
        \centering
        \includegraphics[width=\linewidth]{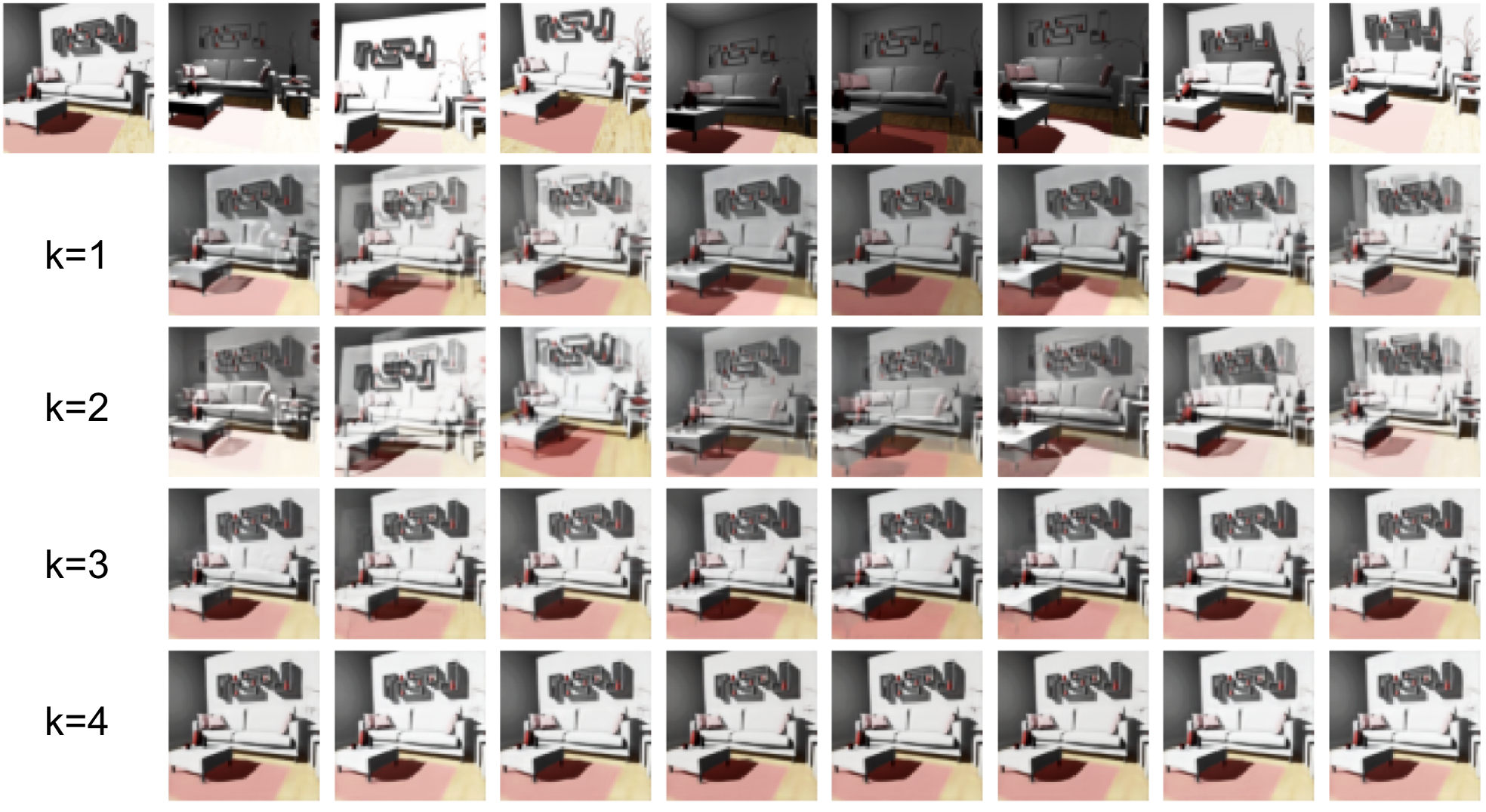}
        \caption{Decomp Diffusion}
        \label{fig:comet_factors}
    \end{subfigure}
    \caption{
    \textbf{Swapping individual latent components across images.}
In each of the subfigures, we fix a reference image $A$ (top left) and a set of other images (top). Below,
each row shows the result of replacing a single latent component of $A$
with the corresponding component from another image, while keeping all
remaining components fixed. The results illustrate how changing individual latent components affects specific aspects of the generated image. The COMET and Decomp Diffusion subfigures exhibit recombinations that are physically incoherent. See Figure~\ref{fig:our_factors} for a visual comparison illustrating
how our method behaves under the same latent component swaps.
    }
    \label{fig:old_factors}
\end{figure}

\begin{figure}[t]
    \centering
    \includegraphics[width=0.95\columnwidth]{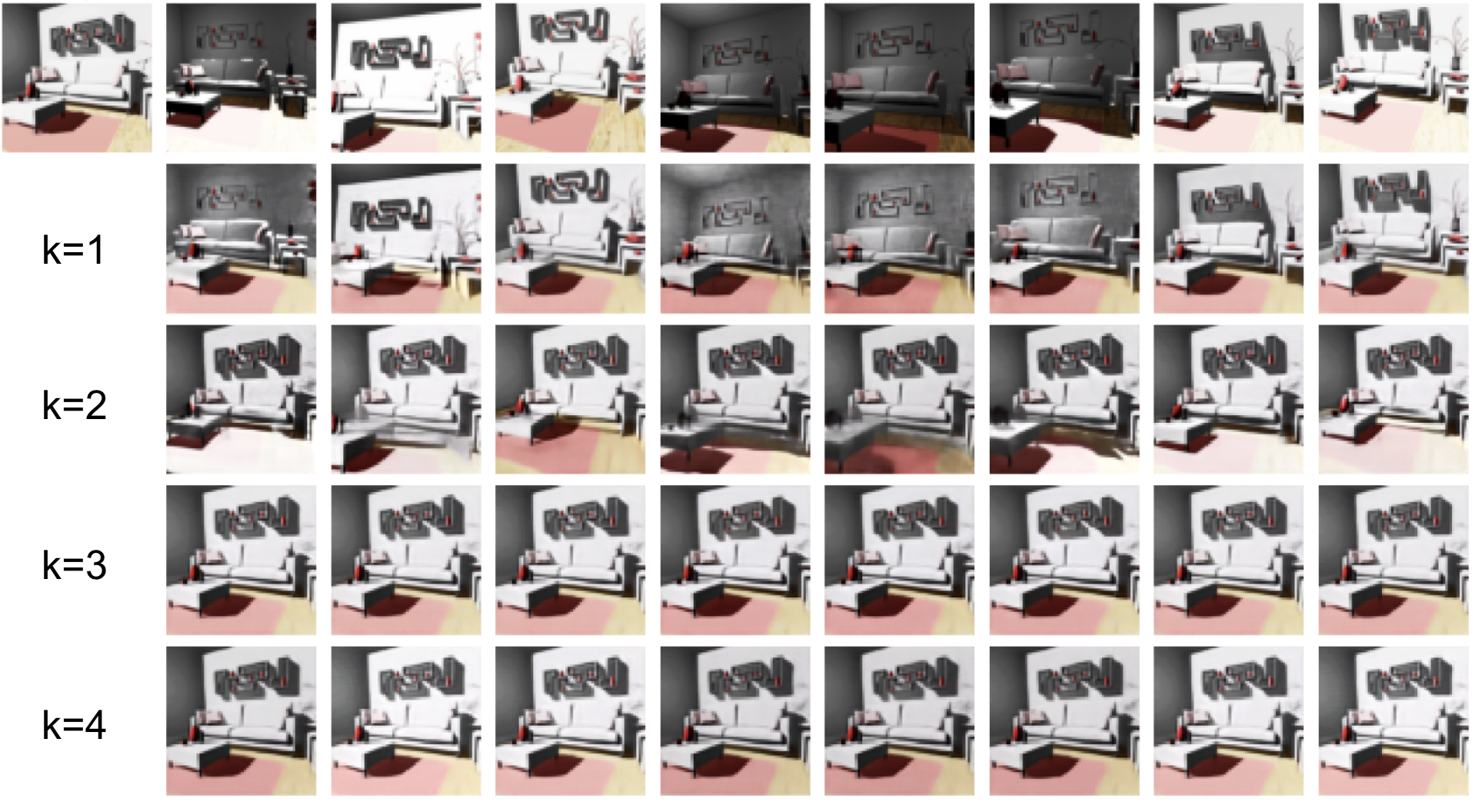}
    \caption{
    \textbf{Latent component swaps for our method.}
    Same visualization protocol as Figure~\ref{fig:old_factors}. In contrast, our recombinations remain physically coherent, with lighting
and camera-related factors varying independently without introducing
geometric or illumination inconsistencies.
    }
    \label{fig:our_factors}
\end{figure}

\begin{figure}[hbt!]
    \centering
    \begin{subfigure}{0.48\columnwidth}
        \centering
        \includegraphics[width=\linewidth]{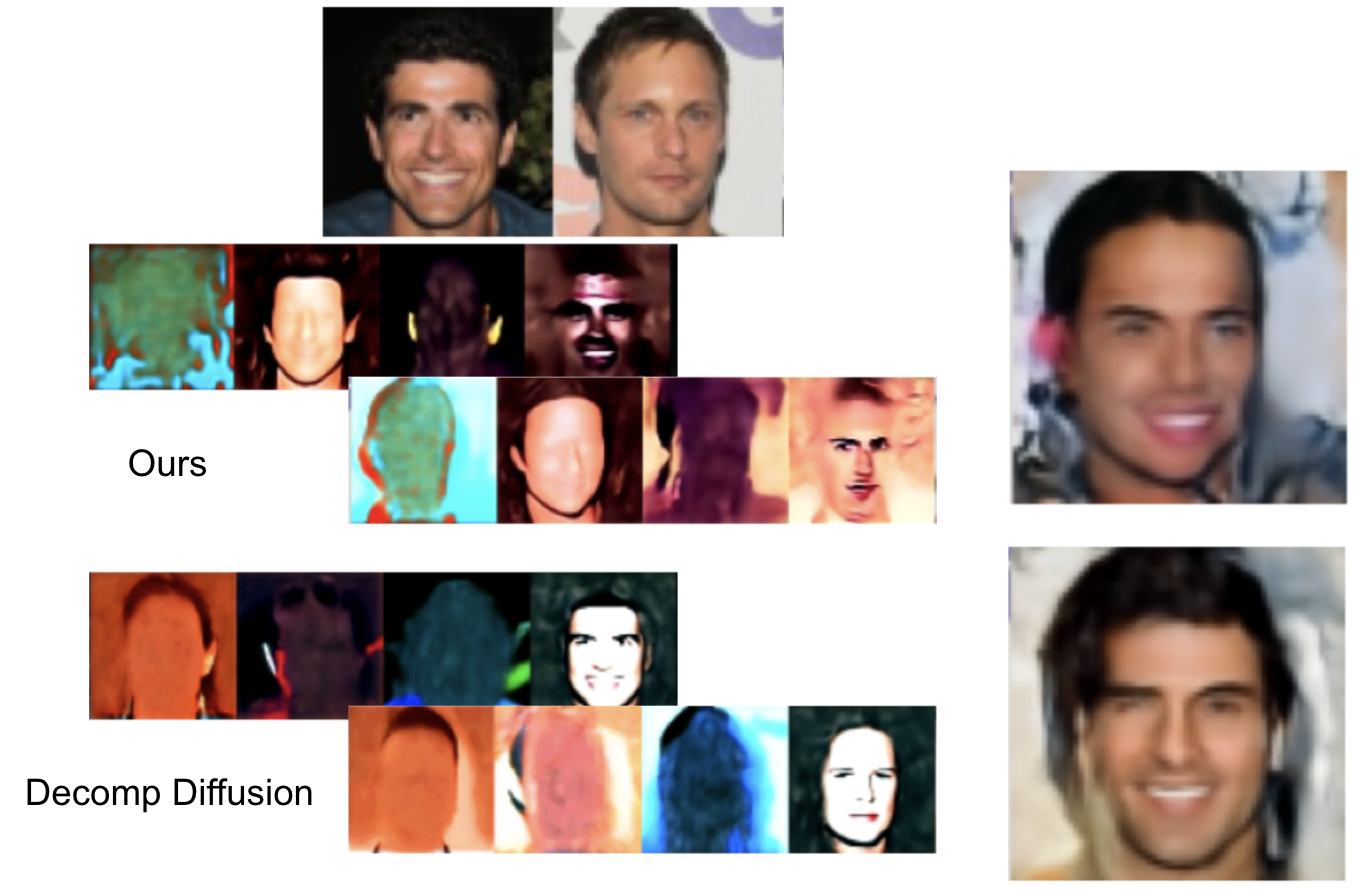}
        \caption{}
        \label{fig:fid_celebahq}
    \end{subfigure}
    \hfill
    \begin{subfigure}{0.48\columnwidth}
        \centering
        \includegraphics[width=\linewidth]{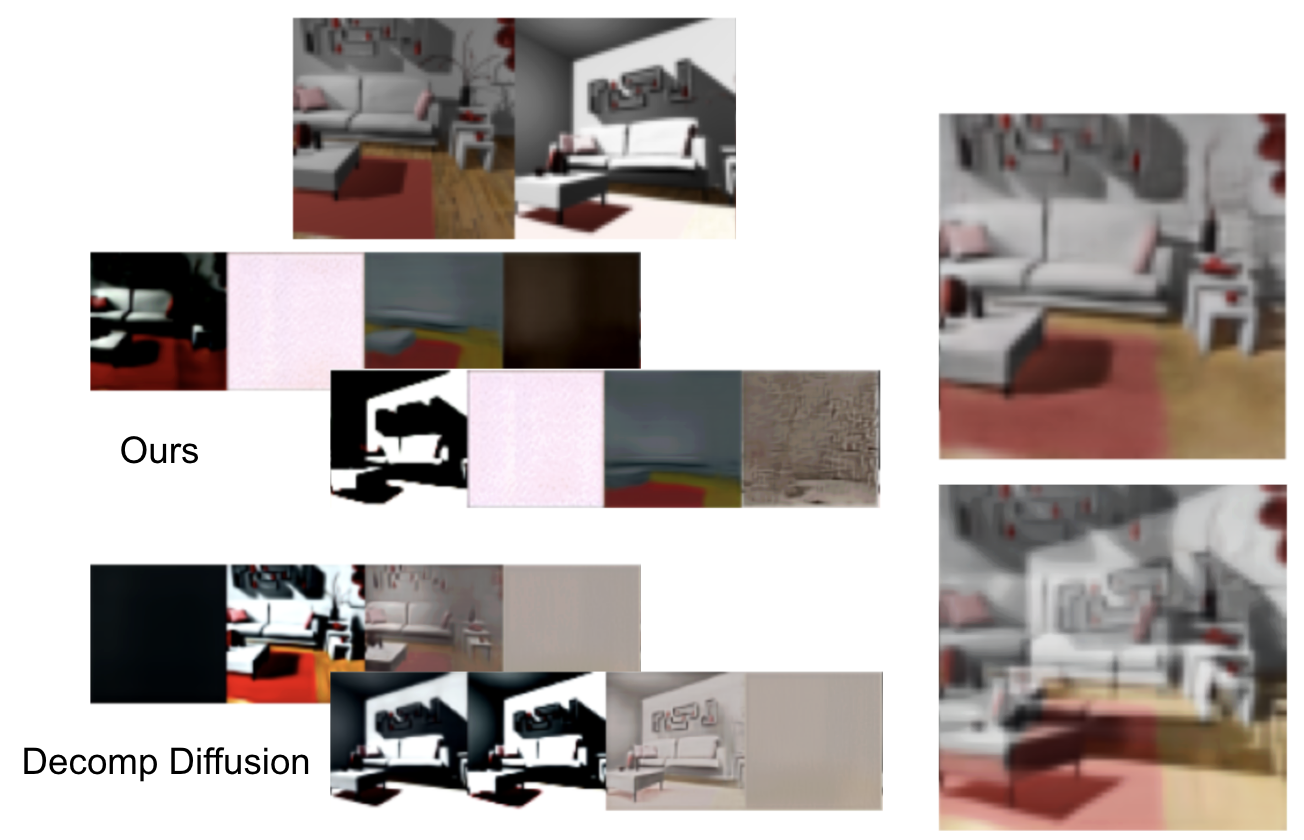}
        \caption{}
        \label{fig:fid_falcor3dv2}
    \end{subfigure}
    \caption{
    The unsupervised training of the baseline compositional model lacks an explicit mechanism to encourage compositional generalization. Recombining latent factors from two source images can therefore produce visually and physically inconsistent artifacts. For example, in (b), camera-related information is distributed across multiple latents, leading to inconsistencies under recombination.
    }
    \label{fig:failed_disentanglement}
\end{figure}

\section{Video Planning}
\subsection{Video Training Details} The model comprises a 3D convolutional video encoder and a latent-conditioned 3D UNet
diffusion model. The video encoder consists of an initial $3\times3\times3$ convolution
followed by $3$ encoding stages. Each stage applies a residual 3D convolutional block
and performs spatial downsampling via a strided $3\times3\times3$ convolution with stride
$(1,2,2)$, doubling the channel width at each stage. Starting from $64$ channels, the
encoder increases the channel dimension to $512$ while preserving the temporal dimension.
The resulting feature map is mean-pooled over time and flattened, then projected by a
fully connected layer to a latent vector of dimension $K \cdot d$, where $K$ is the number
of latent components.

Video generation is performed by a 3D UNet diffusion model with channel multipliers
$(1,2,4)$ and one residual block per resolution level. The UNet operates on inputs of
shape $(B, C, F, H, W)$ and includes attention blocks at selected spatial resolutions.
Diffusion timestep embeddings are computed via sinusoidal embeddings followed by a
two-layer MLP. The latent vector is reshaped into $K$ components and concatenated with
the timestep embedding, and this combined embedding is injected into every residual block
of the UNet. For compositional decoding, the input and conditioning embeddings are
replicated across components, denoising predictions are computed independently per
component, and the final output is obtained by averaging predictions across components.
First-frame conditioning is implemented by concatenating the initial RGB frame across the
temporal dimension at the UNet input.

\begin{figure}[hbt!]
    \centering
     \includegraphics[width=0.9\columnwidth]{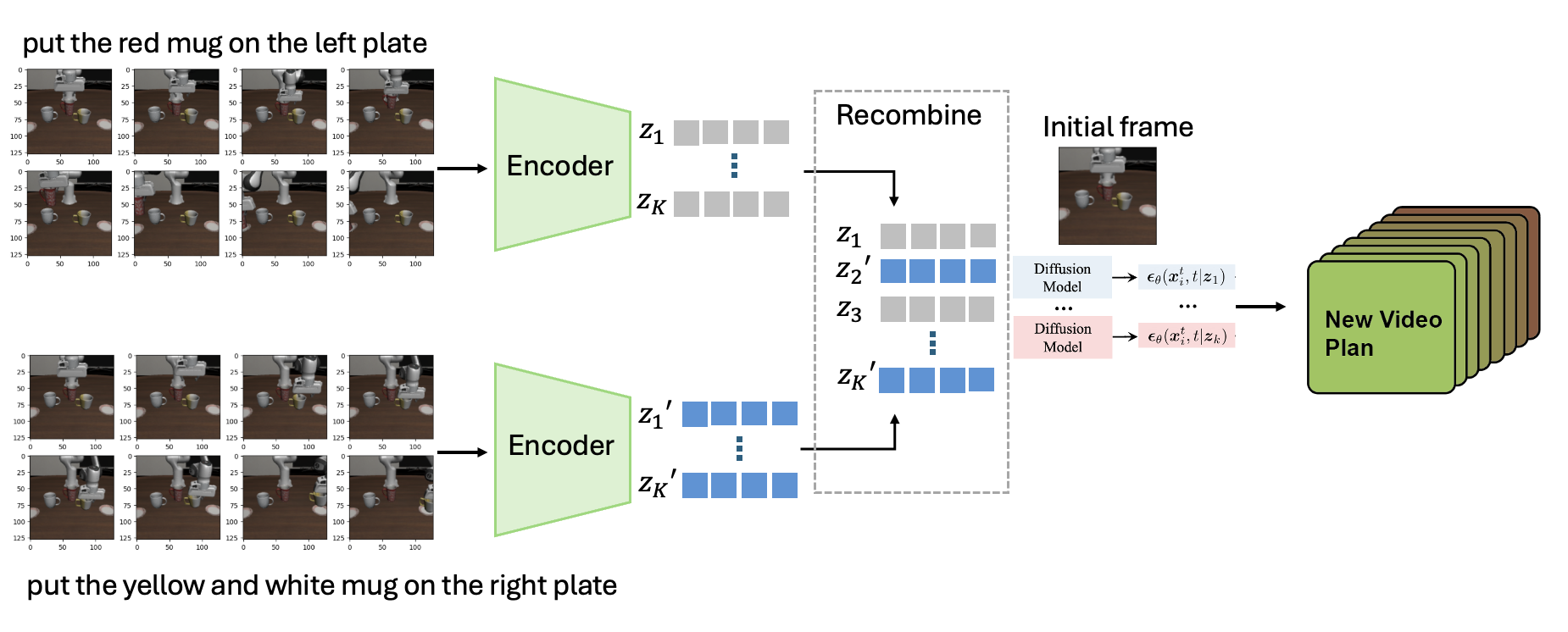}  
     \caption{Demonstration videos are encoded into factorized latent components. Latents corresponding to different action primitives or object interactions are recombined, then decoded by a first-frame-conditioned diffusion model to generate novel, task-consistent video plans.}
    \label{fig:video_arch}
\end{figure}

\subsection{Discriminator Architecture Details}

The discriminator model is a convolutional neural network designed for processing temporally stacked image frames. It employs a series of pseudo-3D convolutional layers that decompose spatiotemporal feature extraction into separate spatial and temporal convolutions. Each convolutional block consists of a spatial \texttt{Conv2D} operation followed by a temporal \texttt{Conv1D} operation, enabling efficient modeling of both spatial structures and temporal dependencies.

The input to the model is a sequence of $F=7$ frames, each with $C=3$ color channels, forming a five-dimensional tensor of shape $(B, C, F, H, W)$. The frames are processed by four convolutional stages, where the number of feature channels increases progressively from 64 to 512. Each stage applies a spatial convolution with a $3 \times 3$ kernel and a stride of $(1,2,2)$, reducing the spatial resolution while maintaining the temporal dimension. The spatial convolutions are followed by temporal \texttt{Conv1D} operations with a kernel size of 3, capturing temporal dependencies across adjacent frames.

After the final convolutional stage, the feature map has $512$ channels and a reduced spatial size of $(H/16, W/16)$. The output is then flattened and passed through a fully connected layer with $512 \times 7 \times (H/16) \times (W/16)$ input dimensions, mapping to a single scalar output. The model does not use batch normalization and employs \texttt{LeakyReLU} activations ($\alpha=0.2$) after each convolutional layer. The final output represents a classification score, which can be used for adversarial training or other video-based tasks.

\subsection{Libero}

\paragraph{Scene 5 and 6.} 
\label{appendix:libero}
Scene 5 contains mug placement tasks in which the robot is instructed to place a mug onto a specified plate. The task variants include: ``Put red mug on left plate", ``Put red mug on right plate", ``Put the white mug on the left plate", and ``Put the yellow and white mug on the right plate". Each task differs only in the target plate location and/or the mug's appearance, while the overall manipulation structure remains the same.

Scene 6 contains object placement tasks involving a plate as a reference target. The task variants include: ``Put chocolate pudding to left of plate", ``Put chocolate pudding to right of plate", ``Put the red mug on the plate", and ``Put the white mug on the plate". These tasks differ in object type and spatial relation relative to the plate.

Given a generated video sequence \( \hat{X} = (\hat{x}_1, \hat{x}_2, \dots, \hat{x}_N) \), we extract a corresponding action sequences
$a_t = \pi_{\text{diff}}(\hat{x}_t, \hat{x}_{t+1})$, for all $t \in \{1, \dots, N-1\}$ where \( \pi_{\text{diff}} \) is the trained diffusion policy, and \( a_t \) is the extracted actions between consecutive frames. We evaluate the quality of generated robotic video rollouts along two axes: physical consistency with the environment dynamics and diversity of explored behaviors.
To assess physical realism, we measure the discrepancy between executed vs generated frames after executing inferred actions. These actions are executed in the Libero environment, producing an actual execution trajectory \( X^E = (x^E_1, x^E_2, \dots, x^E_N) \).

For each generated rollout, we apply an inverse dynamics or action-inference model to recover a sequence of actions from the generated frames, and compute a loss $L_{\text{phys}}$ between the generated frames and the actual frames generated by these inferred actions (which we get from running the simulator). We use the pretrained diffusion policy model (inverse dynamics) from the AVDC codebase, trained on LIBERO Scenes 5–6 demonstration data, and keep it fixed across all methods. Architecturally, $\pi_{\mathrm{diff}}$ encodes the RGB observation with an adapted ResNet-18 with group normalization and spatial softmax pooling, and uses a 1D-convolutional FiLM U-Net diffusion backbone conditioned on the observation embedding \cite{avdc}. The loss between generated and executed frames is $L_{\text{phys}} = \frac{1}{N} \sum_{t=1}^{N} \| x^E_t - \hat{x}_t \|_2^2$ where \( x^E_t \) is the real frame from the Libero execution, \( \hat{x}_t \) is the generated recombined frame, and \( N \) is the total number of frames. 
Lower values of $L_{\text{phys}}$ indicate better adherence to the underlying physical dynamics.
To quantify exploration breadth, we measure coverage in discretized joint-state space. We partition the continuous robot state into bins of resolution $\Delta$ and compute the number of bins visited by generated rollouts.
The resulting metric $\mathrm{Cov}_{\Delta}$ captures the diversity of reachable states induced by recombination, with higher values indicating broader exploration. Unless otherwise stated, all hyperparameters (including bin resolution $\Delta$ and action inference models) are held fixed across methods.\looseness=-1

\subsection{Additional Results}
\label{appendix:video-results}

\begin{figure*}[ht]
    \centering
\includegraphics[width=0.9\textwidth]{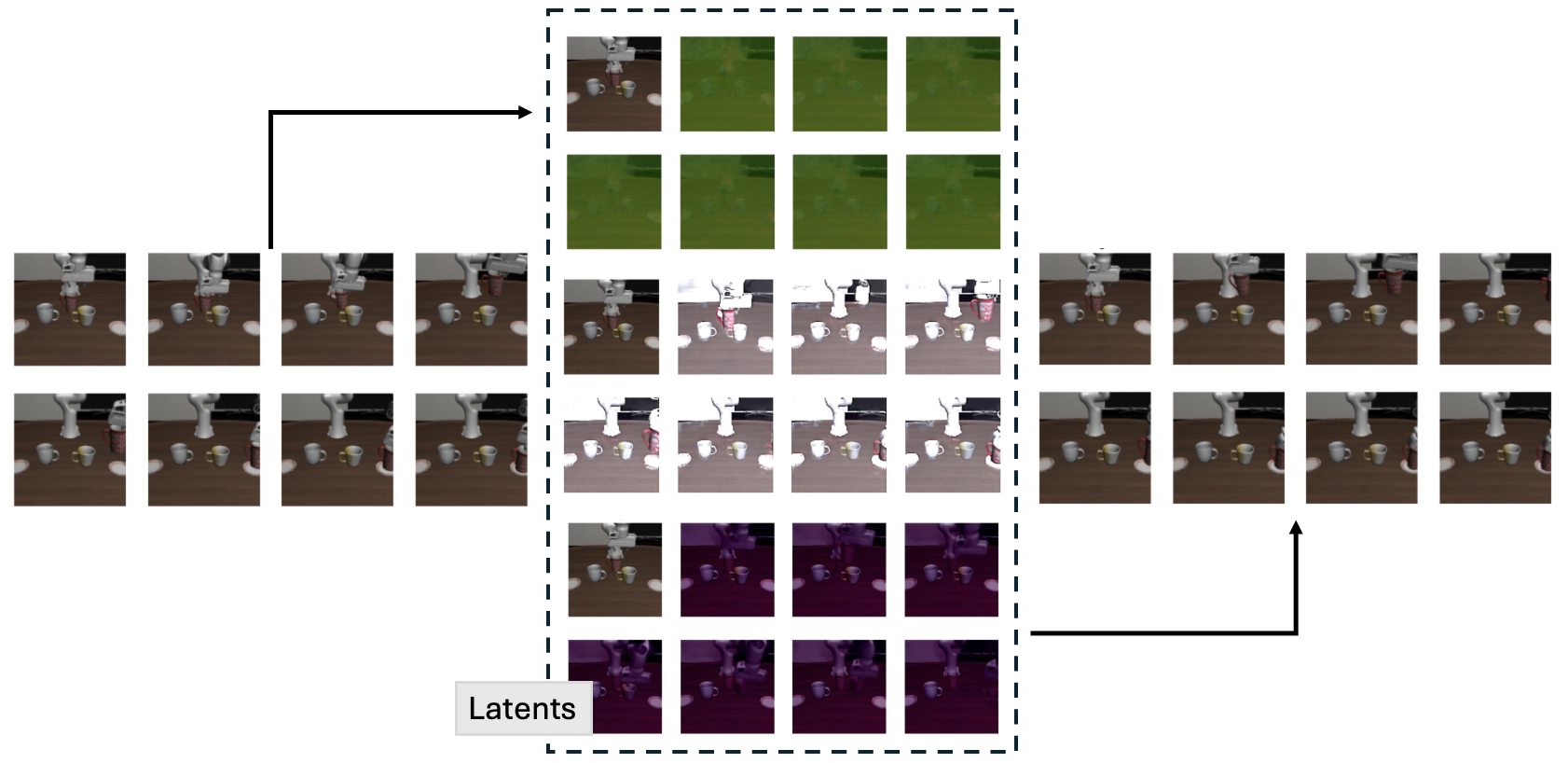} 
    \caption{$\lambda = 0.003$: Good and meaningful representations. Reconstruction quality is fair. See decomposed components and reconstruction quality for large gammas in Figure  \ref{fig:gamma_0.1_decomposition}.} 
\label{fig:gamma_0.003_decomposition}
\end{figure*}

\begin{figure}[hbt!]
    \centering
    \includegraphics[width=0.4\columnwidth]{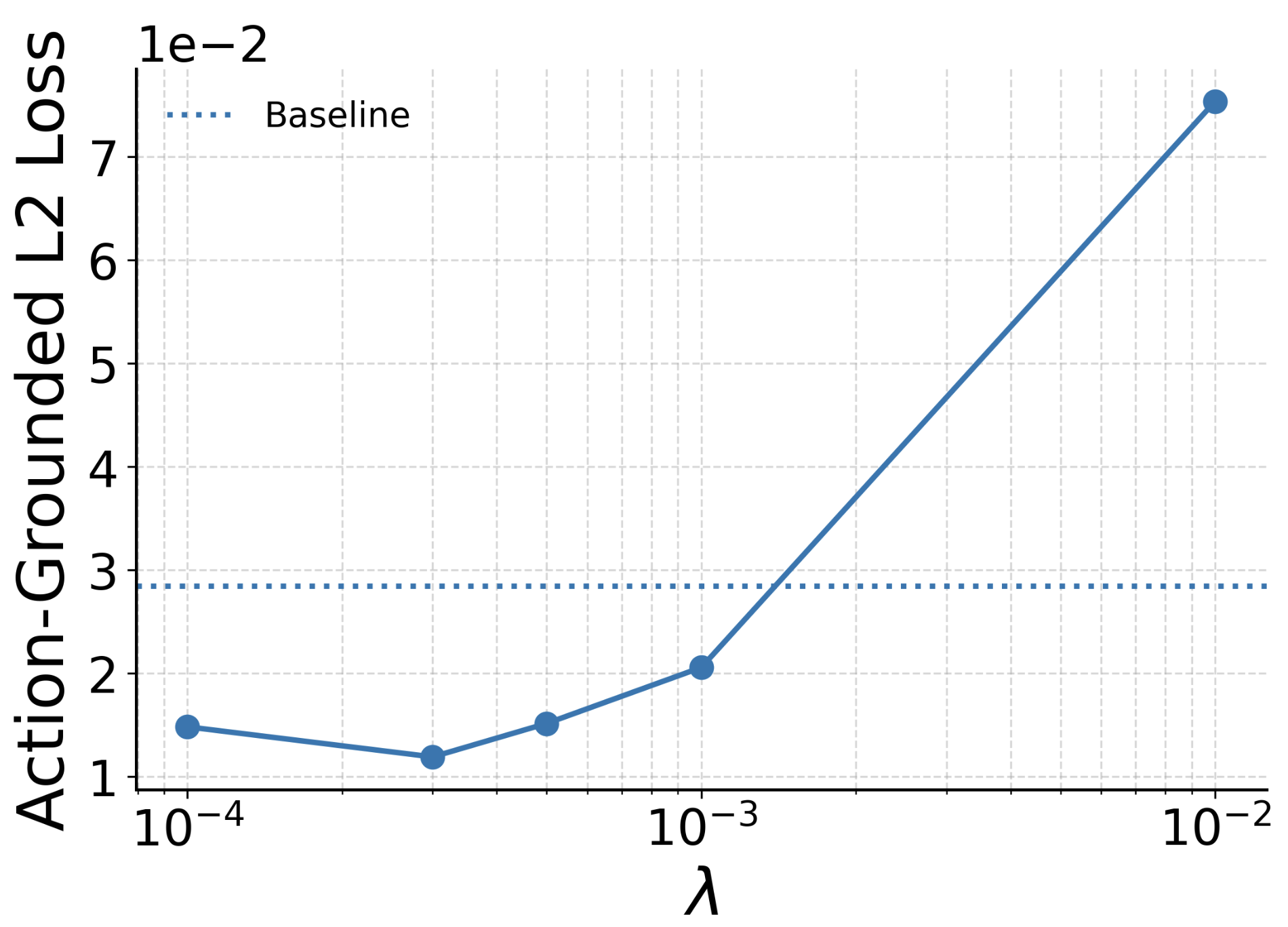} 
    \caption{Effect of discriminator weight strength on Action-Grounded L2 Loss. The loss initially decreases as $\lambda$ increases, reaching an optimal value before rising again, indicating a trade-off between classification guidance and physical realism.
    }
    \label{fig:l2vslambda}
\end{figure}


\begin{figure}[hbt!]
    \begin{subfigure}{\textwidth}
        \centering
        \includegraphics[width=0.9\textwidth]{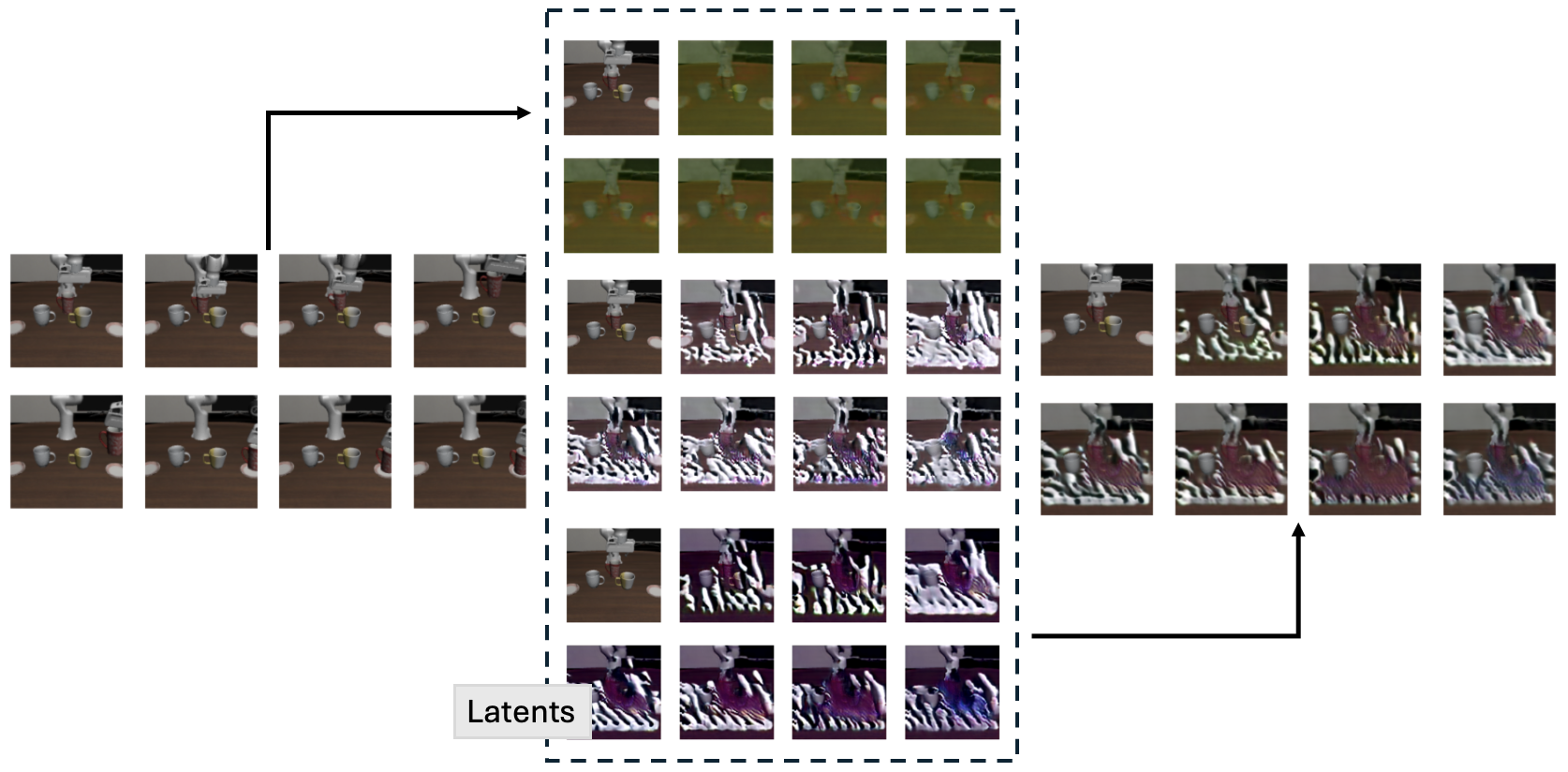} 
        \caption{$\lambda = 0.01$: Decomposed components lose meaning.}
        \label{fig:gamma_0.01_decomposition}
    \end{subfigure}
    \clearpage
    \begin{subfigure}{\textwidth}
        \centering
        \includegraphics[width=0.9\textwidth]{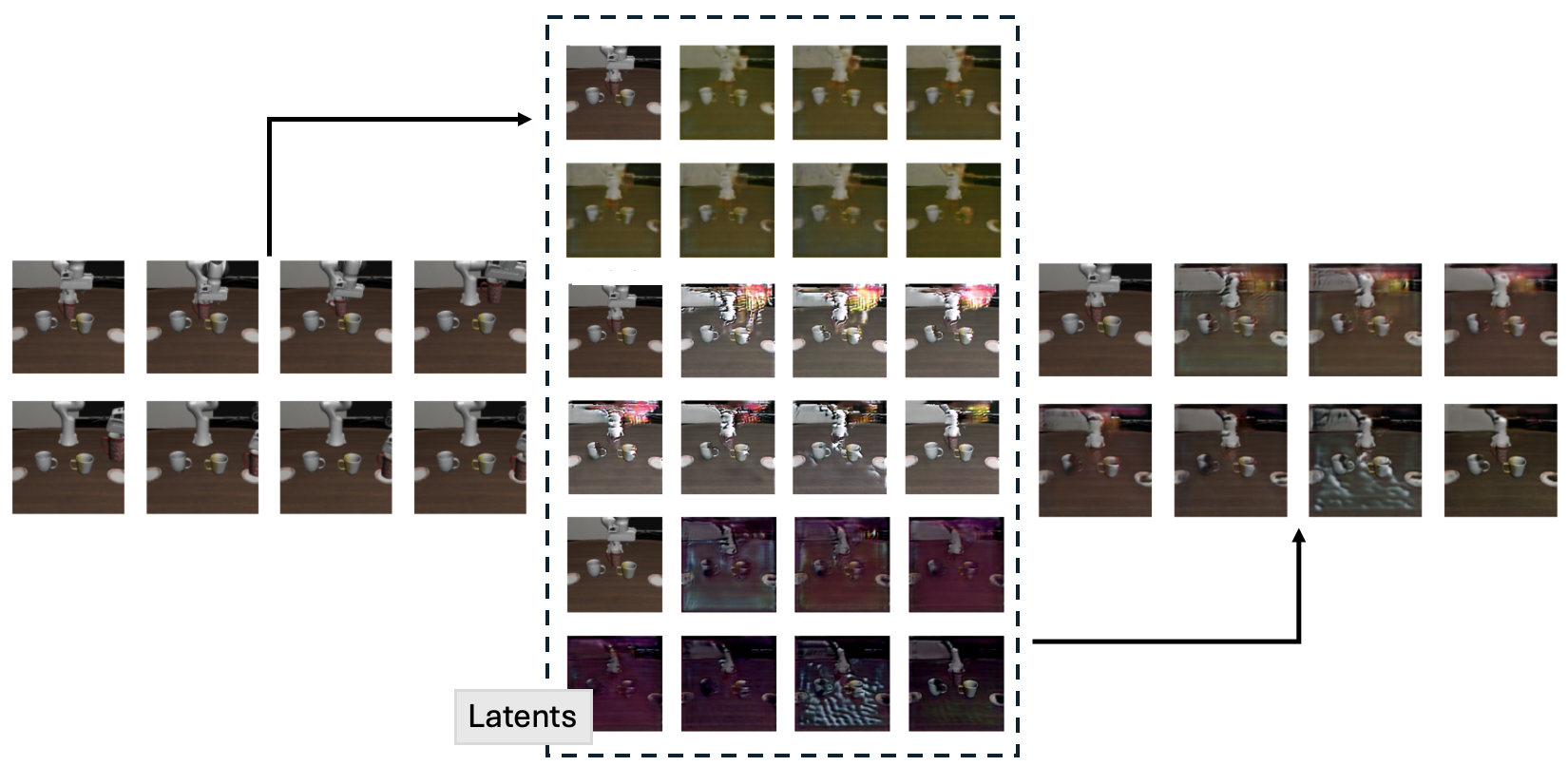} 
        \caption{$\lambda= 0.1$: Decomposed components become noisier.}
        \label{fig:gamma_0.1_decomposition}
    \end{subfigure}
    
    \caption{Decomposition and Reconstruction Quality for different gamma values. For large gammas, the decomposition degrades, leading to noisier reconstructed video plans.}
    \label{fig:gamma_comparison}
\end{figure}

\begin{table}[hbt!]
\centering
\begin{tabular}{lcccccccc}
\toprule
\textbf{Model} 
 & \textbf{Scene5} $\uparrow$ & \textbf{Scene6} $\uparrow$ \\
\midrule
Dataset & 1603 & 1759 \\
Baseline & 9440 & 5400 \\
Baseline + Discriminator ($\lambda$=0.0001) & 9677 & 6994 \\
Baseline + Discriminator ($\lambda$=0.001) & \textbf{12816} & \textbf{9385} \\
Baseline + Discriminator ($\lambda$=0.01) & 7939 & 6358 \\
\bottomrule
\end{tabular}
\caption{Comparison of state-space explored through discretization length = 0.03 from 640 generations. Here, baseline refers to training without discriminator feedback.}
\label{table:state-space}
\end{table}

\subsection{Qualitative assessment}
In Figure \ref{fig:recombine}, we present the decomposition of an input sequence of video frames using a compositional diffusion generative model fine-tuned with a discriminator. By isolating individual components, the model effectively learns to conditionally disentangle distinct concepts based on the initial frame. When analyzing the visualizations of the latent components, we infer that the first latent primarily encodes the static elements of the scene—namely, the two plates and three mugs—while rendering the robot arm indistinct and blurry. The second latent component, in contrast, predominantly whitens the entire scene, selectively preserving only the color of the object being manipulated by the robot arm. The third latent component, beyond capturing variations in shadowing and color, accentuates the dynamic aspects of the scene, particularly tracing the trajectory of the robot arm and its direct interaction with the manipulated object, while diminishing emphasis on the surrounding environment.

\begin{figure}[hb!]
    \centering
    \includegraphics[width=0.9\columnwidth]{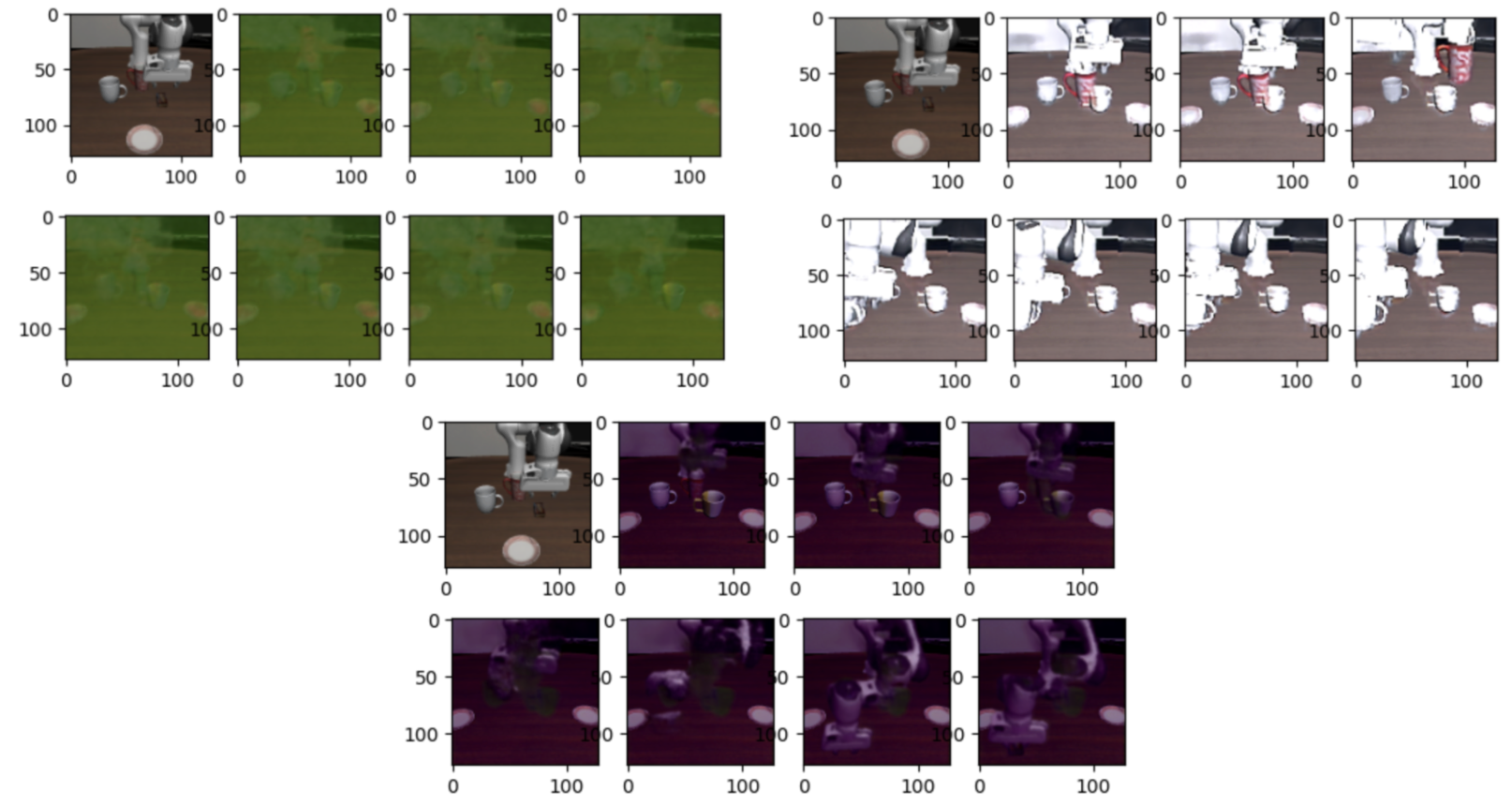} 
    \caption{Visualizations of latent factors for the task 'put the red mug on the left plate'. Each row corresponds to a different recombined trajectory, and columns show selected frames from the rollout. Our interpretation is as follows -- (top left): the robot arm is not visible and only the object configuration is shown. (top right): beginning and end of the executed trajectory. (bottom): intermediate frames from the middle of the trajectory.}
    \label{fig:visualization of latents}
\end{figure}

\begin{figure}[hb!]
    \centering
    \includegraphics[width=0.8\columnwidth]{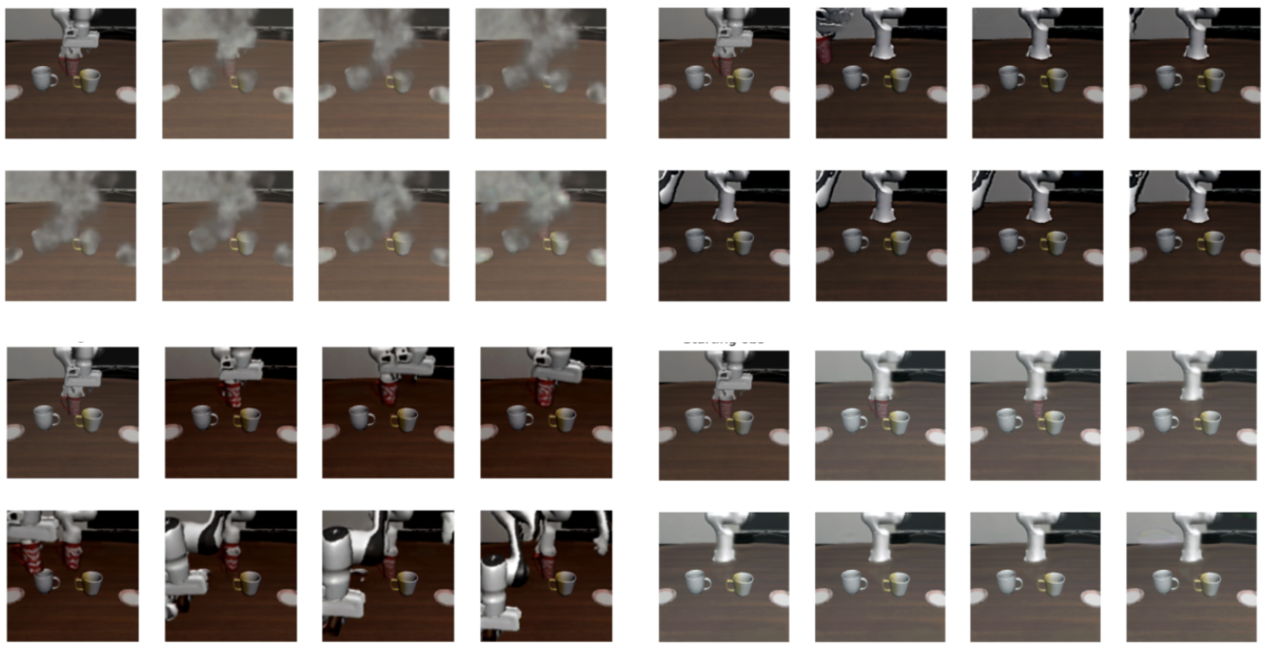} 
    \caption{Visualizations of latents from model trained on four factors. The top right and bottom left emphasize actions near the initial and final interactions between the robotic arm and the red mug. The bottom-right frame primarily captures the static configuration of objects that remain untouched during execution. The top right frame highlights intermediate stages of the trajectory, where the robot arm is actively manipulating the red mug, while background objects are less prominently represented.}
    \label{fig:visualization of latents part 2}
\end{figure}

\newpage

\subsection{State-space coverage and the inverted-U effect}
\label{app:coverage_u_shape}


We discretize executed outcomes into a finite set of state bins $\mathcal{B}$ and let $q_\lambda(b)$ denote the
probability that an executed rollout falls in bin $b$ when sampling is performed with discriminator weight $\lambda$.
If we execute $n$ i.i.d.\ rollouts, the expected number of distinct bins visited is
\begin{equation}
\mathbb{E}[\mathrm{Cov}_n]
\;=\;
\sum_{b\in\mathcal{B}}
\Big(1-\big(1-q_\lambda(b)\big)^n\Big).
\label{eq:expected_coverage_simple}
\end{equation}
For $n\ge 2$, the function $u\mapsto 1-(1-u)^n$ is increasing and concave on $[0,1]$, so
\eqref{eq:expected_coverage_simple} is larger when $q_\lambda$ spreads mass across many bins and smaller when
$q_\lambda$ concentrates on a few.

To relate this to discriminator feedback, partition $\mathcal{B}=\mathcal{F}\cup\mathcal{S}$ into a small set of
\emph{failure} bins $\mathcal{F}$ (e.g., resets/timeouts/trivial end states) and a typically much larger set of
\emph{successful/executable} bins $\mathcal{S}$. In practice, unguided recombination often produces many off-manifold
videos that, after action extraction and execution, collapse into a small number of failures in $\mathcal{F}$.
As $\lambda$ increases from $0$, guidance preferentially suppresses these off-manifold generations, shifting mass in
$q_\lambda$ away from the low-cardinality failure set and into the diverse successful set. This ``filtering'' effect
tends to increase \eqref{eq:expected_coverage_simple} at small-to-moderate $\lambda$.\looseness=-1

For overly large $\lambda$, the sampling distribution can become sharply concentrated on a small set of very
low-energy (``safe''/repetitive) videos. Execution then concentrates $q_\lambda$ on only a few bins inside
$\mathcal{S}$, decreasing \eqref{eq:expected_coverage_simple} by concavity. The combination of initial filtering
(less mass on failures) and eventual concentration among successes yields an inverted-U dependence of coverage on
$\lambda$, with a maximum at an intermediate discriminator weight.

\section{A Theory of Recombination}
\label{appendix:recomb}

\subsection{Additive Energies of Features is Equivalent to Product of Experts}
\label{appendxix: additive energies of features is equivalent to product of experts}

We begin by formalizing a standard equivalence used throughout compositional energy-based modeling. An energy-based model (EBM) defines an unnormalized density $p(x) \propto \exp(-E(x))$ such that $p(x) = \frac{\exp(-E(x))}{Z}$, where $Z=\int \exp(-E(x))\mathrm dx$ denotes the normalizing constant (or the partition function) and $E(x)$ (the energy) is a nonlinear regression function with some latent parameter $\theta$ \cite{lecun2006tutorial, hinton2002training,song2021trainenergybasedmodels, du2020compositional}. We show that additive energies yield a product-of-experts distribution. \nocite{NEURIPS2023_a7a7180f}.

A product-of-experts (PoE) model combines $K$ non-negative expert potentials $\{\tilde{p}_k(x)\}_{k=1}^K$ over the same variable $x\in\mathcal{X}$, via $p_{PoE}(x) \coloneqq \frac1Z \prod_{k=1}^K \tilde{p}_k(x)$, where $Z=\int_{\mathcal{X}}\prod_{k=1}^K \tilde{p}_k(u) \mathrm du$ is the partition function. We show that our discriminator-based feedback can be interpreted as shaping the experts so that recombined expert sets still yield sufficient probability mass in the PoE intersection, which improves the fidelity of the recombination.

\begin{lemma}
    [Additive energies form a PoE distribution]
    Let $\{E_k(x)\}_{k=1}^K$  be a collection of energy functions and $\{\alpha_k\}_{k=1}^K$ be a collection of non-negative coefficients. Define the weighted aggregate energy $E(x) \coloneqq \sum_{k=1}^K \alpha_k E_k(x)$. Next, for $i\in\{1,\dots,K\}$, we define the sequence of unnormalized densities $\tilde{p}_i(x) = \exp(-E_i(x))$. Then, the induced aggregate model is a product of experts $\tilde{p}(x) =  \prod_{k=1}^K \tilde{p}_k(x)^{\alpha_k}$.
\end{lemma}

\begin{proof}Note that
\begin{align*}
\tilde p(x) &= \exp\Big(-\sum_{k=1}^K \alpha_k E_k(x)\Big) \\
&= \prod_{k=1}^K \exp(-\alpha_k E_k(x))
\\
&= \prod_{k=1}^K \tilde p_k(x)^{\alpha_k}.
\end{align*}
To show that $\tilde{p}(x)$ is a realizable unnormalized density, for $k\in[K]$ let $p_k$ be the normalized density corresponding to $\tilde{p}_k$ such that $p_k = \tilde{p}_k(x) / \int \tilde{p}_k(x) \mathrm dx \coloneqq \tilde{p}_k(x)/Z_k$ where $Z_k$ is the partition function of $\tilde{p}_k$. Then, $\tilde{p}_k^{\alpha_k} = (Z_k p_k)^{\alpha_k} = Z_k^{\alpha_k} p_k^{\alpha_k}$. Then, the constants $\prod_{k=1}^K Z_k^{\alpha_k}$ are absorbed into the partition function $Z = \int \prod_k \tilde{p}_k(x)^{\alpha_k}\mathrm dx$, yielding \[p(x) = \frac{\prod_{k=1}^K \tilde{p}_k(x)^{\alpha_k}}{\prod_{k=1}^K Z_k^{\alpha_k}} = \frac{\prod_{k=1}^K \tilde{p}_k(x)^{\alpha_k}}{Z},\]
which proves the lemma.\qedhere\\
\end{proof}

\begin{remark}
    [Naïve recombination artifacts]
    A PoE assigns high probability only to points that satisfy all experts simultaneously. If two experts encode incompatible constraints (for instance, if the camera pose information is split across multiple components) their product can assign near-zero density to the recombined configuration, which introduces artifacts in the recombination. This matches the empirical observation that naïve recombination can create inconsistent samples.\\
\end{remark}

\begin{remark}
    [Additive diffusion] If each expert induces a differentiable density $p_k(x)$, then the PoE score decomposes additively
    $\nabla_x \log p(x) \coloneqq \sum_{k=1}^K \alpha_k \nabla_x \log p_k(x)$ since 
    \[\log p(x) = \sum_k \alpha_k \log p_k(x) + \sum_k \alpha_k \log Z_k - \log Z.\] Therefore, the parameterization in which the denoising network is a sum of per-component contributions can be viewed as learning a PoE over component-conditioned experts.
\end{remark}

\subsection{Recombining Subsets of Latent Representations}
\label{appendix: recombining subsets of latent representations}

We consider the event where arbitrary subset recombinations stay on-manifold, and show that this implies a structural factorization in latent space. To begin, let $\mathcal{M}\subset\mathcal{X}$ denote the data manifold which is the support of $p_{data}$. For $i\in[K]$, let $Z_i \subseteq \mathbb{R}^{d_i}$ denote the latent space corresponding to the $i$'th latent factor. Let $\mathrm{Enc}:\mathcal{M}\to\mathcal{Z}_1\times\cdots\times\mathcal{Z}_K$ be an encoder and $f:\mathcal{Z}_1\times\dots\times\mathcal{Z}_K \to \mathcal X$ be a decoder. Define the set of valid codes $\mathcal{Z}\coloneqq \mathrm{Enc}(\mathcal{M})$ and projections $\pi_k(\mathcal{Z})\subseteq \mathcal{Z}_k$. Next, define the coordinate-wise recombination operator for a mask $S\in\{0,1\}^K$ by \[\mathrm{Mix}_S(z,z')\coloneqq S\odot z + (1-S)\odot z',\]
such that each coordinate is selected from one of the two codes. We then show that closure under coordinate-wise recombination implies that the latent support factorizes as a Cartesian product.

\begin{proposition}
    [Closure under subset recombination implies Cartesian-product support] Assuming  (i) $f(\mathrm{Enc}(x))=x$ for all $x\in\mathcal{M}$ and (ii) for all $z,z' \in\mathcal{Z}$ and all masks $S\in\{0,1\}^K$, $\mathrm{Mix}_S(z,z')\in\mathcal{Z}$,
    we have that the valid code set factorizes as a Cartesian product such that
    $\mathcal{Z} = \pi_1(\mathcal{Z}) \times\cdots\times\pi_K(\mathcal{Z})$. Hence, every recombination of per-component values that appears on the manifold can be recombined without leaving the manifold.
\end{proposition}
\begin{proof}
    Let $a_k \in \pi_k(\mathcal{Z})$ be arbitrary. By definition of the projection, for each $k$ there exists a code $z^{(k)}\in\mathcal{Z}$ whose $k$'th coordinate is equal to $a_k$. Initialize $\bar{z}\leftarrow z^{(1)}$. For $k=2,\dots,K$, define a mask $S^{(k)}$ that selects coordinate $k$ from $z^{(k)}$ and all other coordinates from $\bar{z}$. By closure, $\bar{z}\leftarrow \mathrm{Mix}_{S^{(k)}}(\bar{z}, z^{(k)}) \in \mathcal{Z}$ after each step. At the end, $\bar{z}$ contains coordinates $(a_1,\dots,a_K)$ and lies in $\mathcal{Z}$, which gives $\pi_1(\mathcal{Z})\times\cdots\times\pi_K(\mathcal{Z}) \subseteq \mathcal{Z}$. The reverse inclusion $\mathcal{Z}\subseteq \pi_1(\mathcal{Z})\times\cdots\times\pi_K(\mathcal{Z})$ then follows by definition of our projection operators. Together, this proves the proposition.\qedhere\\
\end{proof}

\begin{remark}
    The above formalizes a stronger notion of independent controllability, where each latent component can vary independently on the learned manifold. This is a structural condition that is consistent with higher disentanglement scores on datasets where ground-truth factor space has an approximate product structure.\\
\end{remark}

\subsection{Contraction of Pairwise Mutual Information through Recombination}

We show that random subset recombination reduces statistical dependence between latent components, in the sense that pairwise mutual information contracts by a factor determined by the probability of selecting both components from
the same source sample.
\nocite{anand2025feelgood} 

\begin{definition}
    [Mutual information] For random variables $X$ and $Y$ with joint distribution $p_{X,Y}$ and marginals $p_X$ and $p_Y$ (respectively), the mutual information between $X$ and $Y$ is given by \[I(X;Y) \coloneqq \mathrm{KL}(p_{X,Y}\|p_X p_Y) = \mathbb{E}_{(X,Y)\sim p_{X,Y}} \left[\log \frac{p_{{X,Y}}(X,Y)}{p_X(X) p_Y(Y)}\right].\]
\end{definition}

To proceed, let $Z=(Z_1,\dots,Z_K)\sim P_Z$ denote a random latent code (e.g., the aggregated posterior of $\mathrm{Enc}(X)$) (as defined in \cref{appendix: recombining subsets of latent representations}), and let $Z^{A}, Z^{B}\sim P_Z$ (i.i.d.). Next, let $S\in\{0,1\}^K$ be a random mask
independent of $Z^A,Z^B$. Define the recombined code $\tilde Z$ component-wise as $\tilde Z_k \coloneqq S_k Z^A_k + (1-S_k) Z^B_k$ where $k=1,\dots,K$. 

\begin{lemma}[Pairwise mutual information contracts under recombination]
\label{lemma :mutual information contracts}
Fix two indices $i\neq j$. Let $\alpha_{ij}\coloneqq  \Pr[S_i=S_j]$. Then, we have that \[\mathcal{I}(\tilde Z_i;\tilde Z_j) \leq \alpha_{ij} \cdot \mathcal{I}(Z_i;Z_j).\] 
\end{lemma}

\begin{proof}
Since $Z^A$ and $Z^B$ are independent, the pair $(\tilde Z_i,\tilde Z_j)$ admits a mixture form where
with probability $\alpha_{ij}$, both coordinates come from the same draw (either $A$ or $B$), yielding the correlated joint
distribution $P_{Z_i,Z_j}$; with probability $1-\alpha_{ij}$ they come from different draws, yielding independence $P_{Z_i}P_{Z_j}$.
Thus, we write
\begin{equation}
P_{\tilde Z_i,\tilde Z_j} = \alpha_{ij} P_{Z_i,Z_j} + (1-\alpha_{ij}) P_{Z_i}P_{Z_j}.
\end{equation}
Since $P_{\tilde Z_i}=P_{Z_i}$ and $P_{\tilde Z_j}=P_{Z_j}$, we can write
\begin{equation}
\mathcal{I}(\tilde Z_i; \tilde Z_j) = \mathrm{KL}\left(P_{\tilde Z_i,\tilde Z_j}\| P_{Z_i}P_{Z_j}\right).
\end{equation}
By convexity of $\mathrm{KL}(\cdot\|\cdot)$ in its first argument,
\begin{align}
\mathcal{I}(\tilde Z_i; \tilde Z_j)
&=\mathrm{KL}\left(\alpha_{ij} P_{Z_i,Z_j} + (1-\alpha_{ij}) P_{Z_i}P_{Z_j} \| P_{Z_i}P_{Z_j}\right)\\
&\leq \alpha_{ij} \cdot \mathrm{KL} \left(P_{Z_i,Z_j}\| P_{Z_i}P_{Z_j}\right)
+ (1-\alpha_{ij}) \cdot \mathrm{KL}\left(P_{Z_i}P_{Z_j}\|P_{Z_i}P_{Z_j}\right)\\
&= \alpha_{ij} \cdot \mathcal{I}(Z_i;Z_j),
\end{align}
which proves the lemma. Therefore, random subset recombination mixes correlated latents with independent latents (from different samples), thereby shrinking dependence. Hence, enforcing that decoded samples from recombined latents remain on-manifold provides an implicit pressure toward latents that are more ``factorial'' (lower dependence), which aligns with common disentanglement objectives.
\qedhere
\end{proof}



\section{Conclusion, Limitations, and Future Work}

\paragraph{Conclusion.} We introduce a discriminator-driven approach to guide compositional learning in both static image recombination and video-based robotic planning. By leveraging a discriminator to distinguish between single-source and recombined data, we demonstrate how its feedback refines latent representations, leading to improved perceptual quality, disentanglement, and exploration efficiency. Our approach extends diffusion-based decomposition methods by incorporating an inductive bias for incentivizing recombined samples to adhere closely to the training data. For video recombination, we applied our method by encoding multi-frame sequences into latent representations while conditioning generation on an initial frame, showing how such an approach can find meaningful latent components from videos, which can be further used to generate new video samples. Our findings highlight the broader potential of incorporating learned discriminators as implicit regularizers in generative modeling.\looseness=-1

\paragraph{Limitations.}
A limitation of our current implementation lies in how discriminator supervision is applied during diffusion training.
The discriminator is trained to distinguish between images generated from single-source latent components and images
generated from recombined components, but these predictions are formed using a shared noisy input originating from
one source image. As a result, some recombined predictions mix latent components from different images while being
conditioned on a noise realization tied to only one of them, which is not fully aligned with a self-consistent
generative process. In practice, the discriminator operates on predicted denoised images and appears to learn a robust notion of
recombination plausibility, which empirically improves sample quality and representation structure. Nonetheless,
a more principled formulation that better aligns noise conditioning with recombined latents may further strengthen
the approach and is an interesting direction for future work. Additionally, in our current implementation, the discriminator is not explicitly conditioned on the diffusion timestep, even though predictions at different noise levels exhibit distinct feature statistics. Incorporating timestep-aware conditioning is a natural extension that may further improve stability and gradient fidelity across the diffusion schedule.

\paragraph{Future Work.} Our method enables extrapolation while remaining close to observed manifold; however, extremely novel compositions may still be suppressed. We identify three main directions of future work. Firstly, it would be interesting to extend our framework to additional domains, such as multi-modal learning or real-world robotic execution \cite{liu2025rdt1bdiffusionfoundationmodel,NEURIPS2023_a7a7180f,pmlr-v247-lin24a,anand2025meanfield}, further enhancing compositional generalization and structured exploration, or using multiple discriminators during training \cite{durugkar2017,10890126}. Secondly, it would be interesting to use a hybrid approach of unsupervised learning and feedback from environment to form novel compositions. Finally, to improve the stability of the algorithm and preventing (unlikely) potential degenerate solutions, it might be interesting to add cycle constraints \cite{10030802} to ensure that the intended components match the source latents.

\end{document}